\documentclass[sigconf]{acmart}
\AtBeginDocument{%
  }

\copyrightyear{2025}
\acmYear{2025}
\setcopyright{cc}
\setcctype{by}
\acmConference[KDD '25]{Proceedings of the 31st ACM SIGKDD Conference on Knowledge Discovery and Data Mining V.2}{August 3--7, 2025}{Toronto, ON, Canada}
\acmBooktitle{Proceedings of the 31st ACM SIGKDD Conference on Knowledge Discovery and Data Mining V.2 (KDD '25), August 3--7, 2025, Toronto, ON, Canada}
\acmDOI{10.1145/3711896.3737025}
\acmISBN{979-8-4007-1454-2/2025/08}

\usepackage{amsmath}
\usepackage{amsthm}
\usepackage{xcolor}
\usepackage{xspace}
\usepackage{bm}
\usepackage{cleveref}
\usepackage[vlined,linesnumbered,ruled]{algorithm2e}
\usepackage[operators,sets]{cryptocode}
\usepackage{physics2} %
\usephysicsmodule{ab}
\usepackage{enumitem}
\usepackage{mathtools}
\usepackage{derivative}

\usepackage{graphicx}
\usepackage{booktabs}
\usepackage{colortbl}
\usepackage{multirow}
\usepackage{threeparttable}

\DeclareMathOperator*{\R}{R}

\usepackage{thmtools, thm-restate}
\theoremstyle{plain}

\declaretheorem{lemma}

\theoremstyle{remark}
\newtheorem{remark}{Remark}

\theoremstyle{definition}
\newtheorem{definition}{Definition}
\newtheorem{assumption}{Assumption}

\DeclareMathOperator*{\E}{\mathbb{E}}
\DeclareMathOperator*{\Var}{\mathrm{Var}}

\DeclareMathOperator*{\Tr}{Tr}

\newcommand{\conlinucbsk}{{CLiSK}\xspace}
\newcommand{\conlinucbme}{{CLiME}\xspace}
\newcommand{\conlinucbskme}{{CLiSK-ME}\xspace}
\newcommand{\doubletilde}[1]{\Tilde{\Tilde{#1}}}
\newcommand\numberthis{\addtocounter{equation}{1}\tag{\theequation}}
\renewcommand{\vec}[1]{\bm{#1}}
\usepackage{pifont}%
\newcommand{\cmark}{\ding{51}}%
\newcommand{\xmark}{\ding{55}}%

\newcommand{\1}{\mathbb{I}}
\newcommand{\inprod}[2]{\left\langle #1, #2 \right\rangle}
\newcommand{\sign}{{\rm sign}}
\begin{document}

\title{Leveraging the Power of Conversations: Optimal Key Term Selection in Conversational Contextual Bandits}

\author{Maoli Liu}
\orcid{0000-0002-6321-6576}
\affiliation{%
  \institution{The Chinese University of Hong Kong}
  \city{Hong Kong}
  \state{}
  \country{China}
}
\email{mlliu@cse.cuhk.edu.hk}

\author{Zhuohua Li}
\authornote{Zhuohua Li is the corresponding author.}
\orcid{0000-0002-1390-0222}
\affiliation{%
  \institution{Guangzhou Institute of Technology, Xidian University}
  \city{Guangzhou}
  \state{Guangdong}
  \country{China}
}
\additionalaffiliation{%
  \institution{The Chinese University of Hong Kong}
  \city{Hong Kong}
  \state{}
  \country{China}
}
\email{zhli@cse.cuhk.edu.hk}

\author{Xiangxiang Dai}
\orcid{0000-0003-0179-196X}
\affiliation{%
  \institution{The Chinese University of Hong Kong}
  \city{Hong Kong}
  \state{}
  \country{China}
}
\email{xxdai23@cse.cuhk.edu.hk}

\author{John C.S. Lui}
\orcid{0000-0001-7466-0384}
\affiliation{%
  \institution{The Chinese University of Hong Kong}
  \city{Hong Kong}
  \state{}
  \country{China}
}
\email{cslui@cse.cuhk.edu.hk}

\begin{abstract}
Conversational recommender systems proactively query users with relevant ``\textit{key terms}'' and leverage the feedback to elicit users' preferences for personalized recommendations.
Conversational contextual bandits, a prevalent approach in this domain,  aim to optimize preference learning by balancing exploitation and exploration.
However, several limitations hinder their effectiveness in real-world scenarios.
First, existing algorithms employ key term selection strategies with insufficient exploration, often failing to thoroughly probe users' preferences and resulting in suboptimal preference estimation.
Second, current algorithms typically rely on deterministic rules to initiate conversations, causing unnecessary interactions when preferences are well-understood and missed opportunities when preferences are uncertain.
To address these limitations, we propose three novel algorithms: \conlinucbsk, \conlinucbme, and \conlinucbskme.
\conlinucbsk introduces \textit{smoothed key term contexts} to enhance exploration in preference learning, \conlinucbme \textit{adaptively initiates conversations} based on preference uncertainty, and \conlinucbskme integrates both techniques.
We theoretically prove that all three algorithms achieve a tighter regret upper bound of $\mathcal{O}(\sqrt{dT\log{T}})$ with respect to the time horizon $T$, improving upon existing methods.
Additionally, we provide a matching lower bound $\Omega(\sqrt{dT})$ for conversational bandits, demonstrating that our algorithms are nearly minimax optimal.
Extensive evaluations on both synthetic and real-world datasets show that our approaches achieve at least a 14.6\% improvement in cumulative regret.
\end{abstract}

\ccsdesc[500]{Information systems~Recommender systems}
\ccsdesc[500]{Theory of computation~Online learning algorithms}
\ccsdesc[500]{Theory of computation~Online learning theory}

\keywords{Conversational Recommendation, Preference Learning, Contextual Bandits, Online Learning}

\maketitle

\newcommand\kddavailabilityurl{https://doi.org/10.5281/zenodo.15490021}

\ifdefempty{\kddavailabilityurl}{}{
\begingroup\small\noindent\raggedright\textbf{KDD Availability Link:}\\
The source code of this paper has been made publicly available at \url{\kddavailabilityurl}.
\endgroup
}

\section{Introduction}
\label{sec:intro}

Recommender systems play a crucial role in applications like movie recommendations, online advertising, and personalized news feeds, where providing relevant and engaging content is essential for user satisfaction. To cater to diverse user interests, recommender systems are designed to interact with users and continuously learn from their feedback. For instance, in product and news recommendations, the system can monitor users' real-time click rates and accordingly refine its recommendations.
Modern recommender systems incorporate advanced online learning techniques to adapt in real time and uncover previously unknown user preferences.

A fundamental challenge in recommender systems is the trade-off between \emph{exploration} (i.e., recommending new items to uncover users' unknown preferences) and \emph{exploitation} (i.e., recommending items that align with users' historical preferences).
Contextual bandits~\cite{li2010contextual} address this trade-off by enabling the system to learn from user interactions continuously while optimizing recommendations without compromising the user experience.
In this framework, each item to be recommended is treated as an ``\textit{arm}'', represented by a feature vector.
At each round, the agent (i.e., the recommender system) recommends an arm to the user based on historical interactions and the context of each arm, and then receives feedback/rewards (e.g., clicks).
The objective of the algorithm executed by the agent is to design an arm recommendation strategy that maximizes cumulative reward (or equivalently, minimizes cumulative regret) over time.

\begin{figure}[thb]
    \centering
    \includegraphics[width=0.99\linewidth]{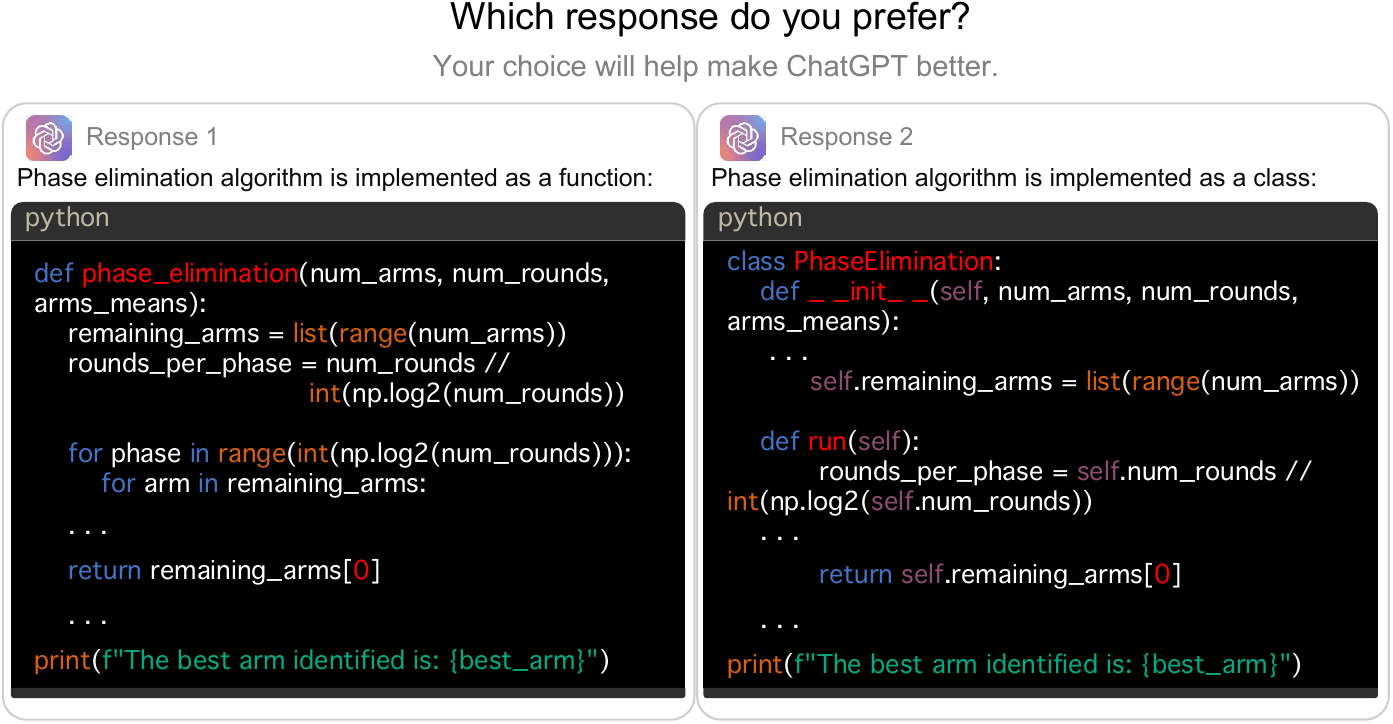}
    \caption{Illustration of conversational recommendation by ChatGPT, where users select their preferred response from presented options.}
    \label{fig:conversation}
\end{figure}

Another major challenge in recommender systems is the ``\textit{cold start}'' problem, where the system initially lacks sufficient data about new users' preferences, making accurate recommendations difficult.
Conversational recommender systems (CRSs)~\cite{christakopoulou2016towards, sun2018conversational, zhang2018towards, gao2021advances} have emerged as a promising solution.
Unlike traditional systems that rely solely on feedback from recommended items, CRSs can actively initiate queries with users to collect richer feedback and quickly infer their preferences.
For example, as shown in~\Cref{fig:conversation}, platforms like ChatGPT 
occasionally present users with multiple response options and allow them to select their preferred one.
Through these interactions, ChatGPT can refine its understanding and improve future responses to better align with user preferences.
To model these interactions, conversational contextual bandits~\cite{zhang2020conversational} are proposed as a natural extension of contextual bandits.
In this framework, besides recommending items (arms) and observing arm-level feedback, the agent can proactively prompt users with questions about key terms and receive key term-level feedback.
The key terms are related to a subset of arms, providing valuable insights into users' preferences and improving recommendation quality.

Despite recent advances in conversational contextual bandits~\cite{xie2021comparison,wang2023efficient,yang2024conversational}, existing approaches still face the following limitations:
\begin{itemize}[leftmargin=*]
    \item \textbf{Insufficient Exploration in Key Term Selection}: Existing studies about conversational bandits fail to sufficiently explore key terms, limiting their effectiveness in preference learning.
    \citet{zhang2020conversational} introduce the ConUCB algorithm with a regret upper bound of \(\mathcal{O}(d\sqrt{T}\log{T})\), where \(d\) is the dimension and \(T\) is the number of rounds.
    However, despite incorporating additional queries about key terms, the method does not yield substantial improvement over non-conversational approaches.
    Since then, improving regret through conversational interactions has remained an open problem in the field.
    \citet{wang2023efficient} and \citet{yang2024conversational} introduce an additional assumption that the key term set spans \(\mathbb{R}^{d}\) and propose the ConLinUCB-BS and ConDuel algorithms, respectively.
    The two algorithms reduce a \(\sqrt{\log{T}}\) term in the regret, but worsen the dependence on \(d\) (as discussed in \Cref{subsec:diss}), resulting in a suboptimal regret bound.
    To achieve optimal regret, more explorative key term selection strategies are needed to efficiently gather informative user feedback and improve learning efficiency.
    \item \textbf{Inflexible Conversation Mechanism}: Existing conversational bandit algorithms~\cite{zhang2020conversational,xie2021comparison} often use a deterministic function to control the frequency of conversations. 
    Specifically, the agent can only initiate \(Q\) conversations at once per \(P\) rounds, where \(P\) and \(Q\) are fixed integers. 
    However, this rigid approach is impractical and insufficient in real-world scenarios.
    For example, in a music streaming service, a fixed-frequency approach may cause unnecessary interactions when users' preferences are already well-understood, disrupting the listening experience.
    Conversely, it may fail to collect feedback when the uncertainty is high, leading to suboptimal recommendations. 
    To address these limitations, a more adaptive conversation mechanism is needed to adjust the interaction frequency based on the preference uncertainty.
\end{itemize}

Motivated by these observations, we develop three algorithms aimed at improving conversational contextual bandits.
To start, we introduce the concept of  ``\textit{smoothed key term contexts}'', inspired by the smoothed analysis for contextual bandits~\cite{kannan2018smoothed}, and propose the \underline{Con}versational \underline{Li}nUCB with \underline{S}moothed \underline{K}ey terms (\conlinucbsk) algorithm.
Specifically, \conlinucbsk launches conversations at a fixed frequency, similar to~\citet{zhang2020conversational}, but greedily selects key terms that are slightly perturbed by Gaussian noise.
For example, in movie recommendations, instead of asking directly about a genre like ``comedy'' or ``drama'',  \conlinucbsk blends elements of related genres, such as ``comedy-drama'' or ``dark comedy''.
This approach helps the system explore users' preferences in a more nuanced manner.
We will show that these small perturbations have \textit{strong theoretical implications}, allowing the agent to explore the feature space more effectively and speed up the learning process.

We next develop the \underline{Con}versational \underline{Li}nUCB with \underline{M}inimum \underline{E}igenvalues (\conlinucbme) algorithm, which introduces an \textit{adaptive conversational mechanism} driven by preference uncertainty.
Unlike the fixed-frequency approach of~\citet{zhang2020conversational,wang2023efficient}, \conlinucbme assesses preference uncertainty and initiates conversations only when the uncertainty is high, thereby maximizing information gain while avoiding unnecessary interactions.
When a conversation is triggered, \conlinucbme selects key terms that target the areas of highest uncertainty within the feature space, rapidly refining user preferences.
This adaptive approach not only ensures that conversations are timely and relevant, but also improves the user experience.
Additionally, we design a family of \textit{uncertainty checking functions} to determine when to assess the uncertainty, offering greater flexibility and better alignment with diverse applications.

The smoothed key term contexts approach in \conlinucbsk and the adaptive conversation technique in \conlinucbme are orthogonal, allowing them to be applied independently or in combination.
Therefore, we further propose the \conlinucbskme algorithm, which integrates both techniques to maximize exploration efficiency and adaptively adjust user interactions.
By leveraging the strengths of both methods, \conlinucbskme enhances exploration efficiency and optimizes user interactions for improved preference learning.

Our algorithms introduce advanced key term selection strategies, significantly enhancing the efficiency of conversational contextual bandits.
Theoretically, we prove that \conlinucbsk achieves a regret upper bound of \(\mathcal{O}(\sqrt{dT\log{T}}+d)\), while \conlinucbme and \conlinucbskme achieve a regret upper bound of \(\mathcal{O}(\sqrt{dT\log{T}})\).
Notably, all three algorithms reduce the dependence on \(T\) by a factor of \(\sqrt{d}\) compared to prior studies.
To the best of our knowledge, our work is the first to achieve the \(\widetilde{\mathcal{O}}(\sqrt{dT})\) regret in the conversational bandit literature.
In addition, we establish a matching lower bound of \(\Omega(\sqrt{dT})\), showing that our algorithms are minimax optimal up to logarithmic factors.

In summary, our contributions are listed as follows.
\begin{itemize}[leftmargin=*]
    \item We propose three novel conversational bandit algorithms: \conlinucbsk with smoothed key term contexts, \conlinucbme with an adaptive conversation mechanism, and \conlinucbskme, which integrates both for improved preference learning.
    \item We establish the minimax optimality of our algorithms by proving regret upper bounds of \(\mathcal{O}(\sqrt{dT\log{T}}+d)\) for \conlinucbsk and \(\mathcal{O}(\sqrt{dT\log{T}})\) for \conlinucbme and \conlinucbskme, along with a matching lower bound of \(\Omega(\sqrt{dT})\). These results underscore the theoretical advancements achieved by our methods.
    \item We conduct extensive evaluations on both synthetic and real-world datasets, showing that our algorithms reduce regret by over 14.6\% compared to baselines.
\end{itemize}

\section{Problem Formulation}
\label{sec:formulation}

In conversational contextual bandits, an agent interacts with a user over \(T\in\mathbb{N}_{+}\) rounds.
The user's preferences are represented by a fixed but \textit{unknown} vector \(\bm{\theta}^{*}\in\mathbb{R}^{d}\),
where \(d\) is the dimension.
The agent's goal is to learn \(\bm{\theta}^{*}\) to recommend items that align with the user's preferences.
There exists a finite arm set denoted by \(\mathcal{A}\), where each arm \(a \in \mathcal{A}\) represents an item and 
is associated with a feature vector \(\bm{x}_a \in \mathbb{R}^d\).
We denote \([T] = \{1,2,\dots, T\}\).
At each round \(t\in[T]\), the agent is given a subset of arms \(\mathcal{A}_t \subseteq \mathcal{A}\). 
The agent then selects an arm \(a_t \in \mathcal{A}_t\) and receives a reward \(r_{a_t,t}\).
The reward is assumed to be linearly related to the preference vector and the feature vector of the arm, i.e., \(r_{a_t,t} =  \bm{x}_{a_t}^{\top}\bm{\theta}^{*} + \eta_t\),
where \(\eta_t\) is a random noise term.

Let \(a_t^*\) be the optimal arm at round \(t\), i.e., \(a_t^* = \arg\max_{a \in \mathcal{A}_t} \bm{x}_{a}^{\top}\bm{\theta}^{*} \).
The agent's objective is to minimize the cumulative regret, which is defined as the total difference between the rewards of the optimal arms and the rewards obtained by the agent, i.e., 
\begin{align*}
\R(T) = 
\sum_{t=1}^{T}\ab(\bm{x}_{a_t^*}^{\top}\bm{\theta}^{*} - \bm{x}_{a_t}^{\top}\bm{\theta}^{*}).
\end{align*}

Beyond observing the user's preference information through arm recommendations, the agent can gather additional feedback by launching conversations involving key terms.
Specifically, a ``key term''  represents a category or keyword associated with a subset of arms.
For example, in movie recommendations, key terms might include genres like ``comedy'' or ``thriller'', and themes such as ``romance'' or ``sci-fi''.
Let \(\mathcal{K}\) denote the finite set of key terms, where each key term \(k \in \mathcal{K}\) corresponds to a context vector \(\tilde{\bm{x}}_k \in \mathbb{R}^d\).
At round \(t\), if a conversation is initiated, the agent selects a key term \(k \in \mathcal{K}\), queries the user, and receives key-term level feedback \(\tilde{r}_{k,t}\).
We follow the formulation of~\citet{wang2023efficient} that the user's preference vector \(\bm{\theta}^{*}\) remains consistent across both arms and key terms.
The relationship between key terms and the user's preference is also linear, i.e., \(\tilde{r}_{k,t} = \tilde{\bm{x}}_{k}^{\top}\bm{\theta}^{*} + \tilde{\eta}_t\),
where \(\tilde{\eta}_t\) is a random noise term.

We list and explain our assumptions as follows.
Both Assumptions~\ref{assumption:normalized-vector} and~\ref{assumption:subgaussian-noise} are consistent with previous works on conversational contextual bandits~\cite{zhang2020conversational,wang2023efficient} and linear contextual bandits~\cite{abbasi2011improved,li2010contextual}.

\begin{assumption} 
    \label{assumption:normalized-vector}
    We assume that the feature vectors for both arms and key terms are normalized, i.e., \(\|\bm{x}_a\|_2 = 1\) and \(\|\tilde{\bm{x}}_k\|_2 = 1\) for all \(a \in \mathcal{A}\) and \(k \in \mathcal{K}\). 
    We also assume the unknown preference vector \(\bm{\theta}^{*}\) is bounded, i.e., \(\|\bm{\theta}^{*}\|_2 \leq 1\).
\end{assumption} 

\begin{assumption}
    \label{assumption:subgaussian-noise} 
    We assume the noise terms \(\eta_t\), \(\tilde{\eta}_t\) are conditionally independent and 1-sub-Gaussian across \(T\) rounds.
\end{assumption}

\section{Algorithm Design}
\label{sec:algorithms}
In this section, we introduce our proposed algorithms, outlining their key components and implementation details.

\subsection{\conlinucbsk Algorithm}
\label{sec:conlinucbsk}

To enhance the exploration of users' preferences, we introduce the \textit{smoothed key term contexts} and propose the \conlinucbsk algorithm, detailed in Algorithm~\ref{alg:conlinucbsk}. 
The algorithm consists of two main modules: key term selection (Lines~\ref{line:smooth-key-term} to~\ref{line:decrease-q}) and arm selection (Lines~\ref{line:update-theta} to~\ref{line:update-b-arm}).
Specifically, in each round \(t\), the agent first determines whether to initiate a conversation based on a predefined query budget (Lines~\ref{line:query-budget} and~\ref{line:query-budget-loop}).
If a conversation is initiated, the agent selects a key term \(k\) (Line~\ref{line:select-key-term}) and queries the user about it.
Subsequently, the agent updates its estimate of the preference vector \(\bm{\theta}_t\) (Line~\ref{line:update-theta}) and selects an arm \(a_t\) for recommendation (Line~\ref{line:select-arm}).
The strategies for key term selection and arm selection are elaborated as follows.

\begin{algorithm}[htb]
    \DontPrintSemicolon
    \SetKwComment{Comment}{$\triangleright$\ }{}
    \SetKwInput{KwInit}{Initialization}
    \KwIn{\(\mathcal{A}\), \(\mathcal{K}\), \(b(t)\), \(\lambda\), \(\{\alpha_t\}_{t>0}\)}
    \KwInit{\(\bm{M}_1 = \lambda \bm{I}_d\), \(\bm{b}_1 = \bm{0}_d\)}
    \For{\(t = 1, \dots, T\) }{
        \(q_t = \lfloor b(t) \rfloor - \lfloor b(t-1) \rfloor\)\; \label{line:query-budget}
            \While{\(q_t > 0\)}{\label{line:query-budget-loop}
                Smooth the key term contexts to get \(\{\doubletilde{\bm{x}}_{k}\}_{k\in \mathcal{K}}\)\; \label{line:smooth-key-term}
                Select a key term \(k = \arg\max_{k \in \mathcal{K}} \doubletilde{\bm{x}}_k^{\top}\bm{\theta}_t\) \; \label{line:select-key-term}
                Query the user's feedback for \(k\)\; \label{line:query-key-term-feedback}
                Receive the key term-level feedback \(\tilde{r}_{k,t}\)\; \label{line:receive-key-term-feedback}
                \(\bm{M}_t = \bm{M}_{t} + \doubletilde{\bm{x}}_{k,t}\doubletilde{\bm{x}}_{k,t}^{\top}\)\; \label{line:update-M-key-term}
                \(\bm{b}_t = \bm{b}_{t} + \tilde{r}_{k,t}\doubletilde{\bm{x}}_{k,t}\)\; \label{line:update-b-key-term}
                \(q_t = q_t -1\)\; \label{line:decrease-q}
            }
        \(\bm{\theta}_t = \bm{M}_t^{-1}\bm{b}_t\)\; \label{line:update-theta}
        Select \(a_t = \argmax_{a \in \mathcal{A}_t} \bm{x}_a^{\top} \bm{\theta}_t  + \alpha_t \|\bm{x}_a\|_{M_t^{-1}}\)\; \label{line:select-arm}
        Ask the user's preference for arm \(a_t\)\; \label{line:ask-arm}
        Observe the reward \(r_{a_t,t}\)\; \label{line:observe-reward}
        \(\bm{M}_{t+1} = \bm{M}_t + \bm{x}_{a_t}\bm{x}_{a_t}^{\top}\)\; \label{line:update-M-arm}
        \(\bm{b}_{t+1} = \bm{b}_t + r_{a_t,t}\bm{x}_{a_t}\)\; \label{line:update-b-arm}
    }
    \caption{\conlinucbsk} \label{alg:conlinucbsk}
  \end{algorithm}

\subsubsection{Intuition Overview}
Building on insights from~\citet{kannan2018smoothed} and~\citet{raghavan2023greedy}, we add small perturbations to the key term contexts to deepen the exploration of users' preferences.
These perturbations increase data diversity and help uncover preferences that might be overlooked when selecting key terms directly. 
For instance, instead of using ``comedy'' alone, variations like ``romantic comedy'' or ``dark comedy'' can reveal more specific preferences.

Below is the formal definition of smoothed key term contexts, where the perturbations are modeled as Gaussian noise.

\begin{definition}[Smoothed Key Term Contexts]
\label{def:smoothed-key-term-contexts}
Given a key term set \(\mathcal{K}\), the smoothed key term contexts are defined as \(\{\doubletilde{\bm{x}}_{k}\}_{k \in \mathcal{K}}\), where \(\doubletilde{\bm{x}}_{k} = \tilde{\bm{x}}_{k} + \bm{\varepsilon}_{k}\) for each \(k \in \mathcal{K}\).
The noise vector \(\bm{\varepsilon}_{k}\) is independently drawn from a truncated multivariate Gaussian distribution \(\mathcal{N}(\bm{0}, \rho^2 \cdot \bm{I}_d)\), where $\bm{I}_d$ is the \(d\)-dimensional identity matrix and \(\rho^2\) controls the level of perturbations.
Each dimension of \(\bm{\varepsilon}_{k}\) is truncated within \([-R, R]\) for some \(R > 0\), i.e., \(|(\bm{\varepsilon}_{k})_j| \leq R, \forall j\in[d]\). 
\end{definition}
\subsubsection{Key Term Selection}
When initiating conversations, the agent no longer selects key terms directly based on their original contexts.
Instead, the agent applies a small random perturbation to each key term's context, as defined in~\Cref{def:smoothed-key-term-contexts} (Line~\ref{line:smooth-key-term}).
It then greedily selects the key term with the highest value under the perturbed contexts, i.e., \(k = \argmax_{k \in \mathcal{K}} \doubletilde{\bm{x}}_k^{\top}\bm{\theta}_t\) (Line~\ref{line:select-key-term}).

\begin{remark}
    Note that the smoothed key term contexts are re-generated for \textit{each conversation}.
    For notational consistency, we use the same notation \(\{\doubletilde{\bm{x}}_{k}\}_{k \in \mathcal{K}}\) to represent the smoothed key term contexts across different conversations.
\end{remark}

\subsubsection{Conversation Frequency} 
Following~\citet{zhang2020conversational}, \conlinucbsk uses a deterministic function \(b(t)\) to regulate the frequency of conversation initiation.
The function \(b(t)\) is monotonically increasing regarding \(t\) and satisfies \(b(0) = 0\).
At round \(t\), the agent initiates \(q(t) = \lfloor b(t) \rfloor - \lfloor b(t-1) \rfloor\) conversations if \(q(t) > 0\);
otherwise, no conversation is conducted.

\subsubsection{Arm Selection}
\conlinucbsk uses the Upper Confidence Bound (UCB) strategy for arm selection, a prevalent method in linear bandits.
At round \(t\), the agent updates its estimated preference vector \(\bm{\theta}_t\) based on both arm-level and key term-level feedback.
This estimation follows a ridge regression framework with regularization parameter \(\lambda\), i.e., \(\bm{\theta}_t = \bm{M}_t^{-1} \bm{b}_t\), with \(\bm{M}_t\) and \(\bm{b}_t\) defined as
\begin{align*}
\bm{M}_t &= \sum_{s=1}^{t-1} \bm{x}_{a_s} \bm{x}_{a_s}^{\top} + \sum_{s=1}^{t} \sum_{k \in \mathcal{K}_{s}} \doubletilde{\bm{x}}_{k} \doubletilde{\bm{x}}_{k}^{\top} +  \lambda \bm{I}_d, \\
\bm{b}_t &= \sum_{s=1}^{t-1} r_{a_s,s} \bm{x}_{a_s} + \sum_{s=1}^{t}\sum_{k \in \mathcal{K}_{s}}\tilde{r}_{k,s} \doubletilde{\bm{x}}_{k},
\end{align*}
where \(\mathcal{K}_{s}\) is the set of key terms selected at round \(s\). \(\bm{M}_t\) is commonly referred to as the covariance matrix.

After the update, the agent selects the arm with the highest UCB value, i.e., \(a_t = \argmax_{a \in \mathcal{A}_t} \bm{x}_a^{\top} \bm{\theta}_t + \alpha_t \|\bm{x}_a\|_{\bm{M}_t^{-1} }\),
where \(\|\bm{x}\|_{\bm{M}}\) denotes the Mahalanobis norm \(\sqrt{\bm{x}^{\top}\bm{M}\bm{x}}\)
and \(\{\alpha_t\}_{t>0}\) are parameters designed to balance the exploration-exploitation trade-off.

\subsection{\conlinucbme Algorithm}
\label{sec:conlinucbme}

To enable more adaptive and flexible conversation initiation, we introduce the \conlinucbme algorithm, detailed in Algorithm~\ref{alg:conlinucbme}.
The \conlinucbme adopts the same arm selection strategy as \conlinucbsk, but it introduces key innovations in determining when to initiate conversations and which key terms to select.
Unlike \conlinucbsk, which follows a deterministic function \(b(t)\) for scheduling conversations, \conlinucbme adaptively determines when to conduct a conversation based on the uncertainty in the preference estimation.

\begin{algorithm}[htb]
    \DontPrintSemicolon
    \SetKwComment{Comment}{$\triangleright$\ }{}
    \SetKwFunction{UncertaintyChecking}{UncertaintyChecking}{}
    \SetKwInput{KwInit}{Initialization}
    \KwIn{\(\mathcal{A}\), \(\mathcal{K}\), \(\lambda\), \(\alpha\), \(\{\alpha_t\}_{t>0}\)}
    \KwInit{\(\bm{M}_1 = \lambda \bm{I}_d\), \(\bm{b}_1 = \bm{0}_d\)}
    \For{\(t = 1, \dots, T\) }{
        \If{\UncertaintyChecking{\(t\)}}{ \label{line:conlinucbme:check-uncertainty}
            Diagonalize \(\bm{M}_t = \sum_{i=1}^d \lambda_{\bm{v}_i} \bm{v}_i{\bm{v}_i}^{\top}\)\; \label{line:conlinucbme:diagonalize}
            \ForEach{\(\lambda_{\bm{v}_i} < \alpha t\)}{ \label{line:conlinucbme:check-lambda}
                \(k = \argmax_{k\in \mathcal{K}}|\tilde{\bm{x}}_k^{\top}\bm{v}_i|\)\;  \label{line:conlinucbme:select-k}
                \(n_k = \lceil(\alpha t -\lambda_{\bm{v}_i})/c_0^2\rceil\)\; \label{line:conlinucbme:compute-nk}
                Schedule \(n_k\) conversations about the key term \(k\) before next uncertainty checking\;\label{line:conlinucbme:query-user}
                Update \(\bm{M}_t\) and \(\bm{b}_t\) accordingly\;
            }
        }
        \(\bm{\theta}_t = \bm{M}_t^{-1}\bm{b}_t\)\; \label{line:conlinucbme:update-theta}
        Select \(a_t = \argmax_{a \in \mathcal{A}_t} \bm{x}_a^{\top} \bm{\theta}_t  + \alpha_t \|\bm{x}_a\|_{M_t^{-1}}\)\; \label{line:conlinucbme:select-arm}
        Ask the user's preference for arm \(a_t\)\; \label{line:conlinucbme:ask-user}
        Observe the reward \(r_{a_t,t}\)\; \label{line:conlinucbme:observe-reward}
        \(\bm{M}_{t+1} = \bm{M}_t + \bm{x}_{a_t}\bm{x}_{a_t}^{\top}\)\; \label{line:conlinucbme:update-m-arm}
        \(\bm{b}_{t+1} = \bm{b}_t + r_{a_t,t}\bm{x}_{a_t}\)\; \label{line:conlinucbme:update-b-arm}
    }
    \caption{\conlinucbme} \label{alg:conlinucbme}
  \end{algorithm}

\subsubsection{Intuition Overview}

The main idea behind \conlinucbme is to 
adaptively initiate conversations based on the current level of uncertainty in the estimated preference and use key terms to explore the uncertain directions effectively.
Specifically, the covariance matrix \(\bm{M}_t\) encodes information about the feature space, where its eigenvectors represent the principal directions within the space, and the corresponding eigenvalues indicate the level of uncertainty along these directions.
A smaller eigenvalue indicates a higher uncertainty in the associated direction.
Therefore, by guiding the agent to explore such high-uncertainty directions, the agent can reduce uncertainty and improve learning efficiency. 
If the minimum eigenvalue of \(\bm{M}_t\) remains above a certain value, the agent ensures sufficient exploration of the feature space.
To facilitate exploration, we introduce the following assumption.
\begin{assumption}
\label{assumption:c0-for-key-term}
    We assume that the elements in the key term set \(\mathcal{K}\) are sufficiently rich and diverse, such that for any \(\bm{x}\in \mathbb{R}^d\) satisfying \(\|\bm{x}\|_2 = 1\), there exists a key term \(k \in \mathcal{K}\) such that \(|\tilde{\bm{x}}_k^{\top}\bm{x}| \geq c_0\), where \(c_0\) is some constant close to 1. 
\end{assumption}

This mild assumption ensures that the key term set \(\mathcal{K}\) is comprehensive enough to cover all relevant directions in the feature space.
In other words, for any direction \(\bm{x}\) that the agent might need to explore, there exists a key term \(k \in \mathcal{K}\) whose context \(\tilde{\bm{x}}_k\) aligns sufficiently well with \(\bm{x}\).
This diversity allows the agent to effectively reduce uncertainty by exploring underrepresented directions, thereby improving preference learning.

\subsubsection{Conversation Initiation and Key Term Selection}

In \conlinucbme, conversation initiation and key term selection are designed to maximize the information gained from user interactions.
As shown in Algorithm~\ref{alg:conlinucbme}, the agent first evaluates the eigenvalues of the covariance matrix \(\bm{M}_t\) (Line~\ref{line:conlinucbme:diagonalize}).
If any eigenvalue \(\lambda_{\bm{v}_i}\) falls below a certain threshold (derived from \Cref{subsec:regret-conlinusbme}), i.e., \(\lambda_{\bm{v}_i} < \alpha t\) (Line~\ref{line:conlinucbme:check-lambda}), the agent prompts \(n_k = \lceil(\alpha t -\lambda_{\bm{v}_i})/c_0^2\rceil\) conversations by selecting key terms that most closely align with the corresponding eigenvector \({\bm{v}_i}\) (Lines~\ref{line:conlinucbme:select-k} to~\ref{line:conlinucbme:query-user}). 
Here, \(\alpha\in(0,c_0^2)\) is an exploration control parameter that regulates the exploration level.
Note that the agent can distribute these \(n_k\) conversations across multiple rounds before re-evaluating the eigenvalues of the covariance matrix.

To further enhance flexibility and accommodate diverse real-world applications, 
we design an uncertainty checking function \texttt{UncertaintyChecking(\(t\))} (Line~\ref{line:conlinucbme:check-uncertainty}).
This function determines when to assess uncertainty and potentially trigger conversations. 
Examples of such checking functions are given as follows. 
\begin{itemize}[leftmargin=*]
    \item \textbf{Continuous Checking}: The agent assesses uncertainty at every round and initiates conversations as needed.
    \item \textbf{Fixed Interval Checking}: The agent assesses uncertainty every \(P\) rounds, where \(P\) is a fixed integer.
    \item \textbf{Exponential Phase Checking}: The agent evaluates uncertainty at exponentially increasing intervals of \(2^{i}\), where \(i=1,2,\dots\).
\end{itemize}

\begin{remark}
     The uncertainty checking functions in \conlinucbme differ fundamentally from the frequency function \(b(t)\) in ConUCB~\cite{zhang2020conversational}.
     Specifically, these checking functions regulate how often uncertainty is assessed but do not directly dictate conversation initiation.
     In contrast, \(b(t)\) deterministically controls both the timing and number of conversations.
    \conlinucbme and ConUCB also differ in how they select key terms, further distinguishing the two approaches.  

\end{remark}

\begin{remark}
    It is worth noting that the smoothed key term contexts approach in \conlinucbsk and the adaptive conversation technique in \conlinucbme are orthogonal.
    The two strategies can operate independently or be integrated to enhance learning efficiency further.
    To this end, we introduce the \conlinucbskme algorithm, detailed in \Cref{app:conlincubskme}, which integrates both approaches to leverage their complementary strengths.
\end{remark}

\section{Theoretical Analysis}\label{sec:theorem}

This section presents the theoretical results of our algorithms, which employ analytical techniques that differ from standard linear bandit methods.
Detailed proofs of all lemmas and theorems are provided in the Appendices.

\subsection{Regret  Analysis of \conlinucbsk Algorithm}

Following~\citet{zhang2020conversational} and~\citet{wang2023efficient}, we assume \(b(t)=bt\) for some \(b\in(0,1)\).
We start with~\Cref{lemma:smoothed-diff-estimate-true-reward}, which bounds the difference between the estimated and true rewards for each arm.

\begin{restatable}{lemma}{restatelemmasmootheddiffestimatetruereward}
    \label{lemma:smoothed-diff-estimate-true-reward}
    Under Assumptions~\ref{assumption:normalized-vector} and~\ref{assumption:subgaussian-noise}, for \conlinucbsk,
    for any round \(t \in [T]\) and any arm \(a \in \mathcal{A}\), with probability at least \(1-\delta\) for some \(\delta \in (0,1)\), we have
    \begin{align*}
    \ab|\bm{x}_a^{\top} \bm{\theta}_t - \bm{x}_a^{\top} \bm{\theta}^{*}| \leq \alpha_t \|\bm{x}_a\|_{\bm{M}_t^{-1}},
    \end{align*}
    where \(\alpha_t = \sqrt{2\log{(\frac{1}{\delta})}+d\log\ab(1+\frac{t + \ab(1 + \sqrt{d}R)bt}{\lambda d})} + \sqrt{\lambda}\).
\end{restatable}

Next, we examine the smoothed key term contexts and their impact on exploring the feature space.

\begin{restatable}{lemma}{restatelemmasmoothedeigenvaluelowerbound}
    \label{lemma:smoothed-eigenvalue-lower-bound}
    For any round \(t \in [T]\), with the smoothed key term contexts in~\Cref{def:smoothed-key-term-contexts}, \conlinucbsk has the following lower bound on the minimum eigenvalue of the matrix \(\E[\doubletilde{\bm{x}}_k \doubletilde{\bm{x}}_k^{\top}]\) for any \(k \in \mathcal{K}_t\),
    i.e.,
    \begin{align*}
    \lambda_{\min}\ab(\E[\doubletilde{\bm{x}}_k \doubletilde{\bm{x}}_k^{\top}]) \geq c_1 \frac{\rho^2}{\log{|\mathcal{K}|}} \triangleq \lambda_{\mathcal{K}},
    \end{align*}
    where \(c_1\in(0,1)\) is some constant.
\end{restatable}
\Cref{lemma:smoothed-eigenvalue-lower-bound} provides a lower bound on the minimum eigenvalue of the expected outer product of the selected key term.
Intuitively, this implies that under smoothed contexts, the selected key terms exhibit sufficient diversity in the feature space, ensuring that each query contributes meaningful information about the user’s preferences.

\begin{restatable}{lemma}{restatelemmasmoothedlineart}
    \label{lemma:smoothed-linear-t}
    For \conlinucbsk, with probability at least \(1-\delta\) for some \(\delta \in (0,1)\), if \(t\geq T_0\triangleq \frac{8(1+\sqrt{d}R)^2}{b\lambda_{\mathcal{K}}}\log\ab(\frac{d}{\delta})\), we have
    \begin{align*}
        \lambda_{\min}\ab(\sum_{s=1}^{t}\sum_{k \in \mathcal{K}_{s}} \doubletilde{\bm{x}}_k \doubletilde{\bm{x}}_k^{\top}) \geq \frac{\lambda_{\mathcal{K}}bt}{2}.
    \end{align*}
\end{restatable}

\Cref{lemma:smoothed-linear-t} establishes a lower bound on the minimum eigenvalue of the Gram matrix that grows linearly with time \(t\).
This guarantees that \conlinucbsk accumulates enough statistical information to effectively estimate the user's preference vector through ridge regression.
Following these results, we bound  \(\|\bm{x}_a\|_{\bm{M}_t^{-1}}\) in ~\Cref{lemma:smoothed-bounded-x-norm} and derive a high-probability regret upper bound for \conlinucbsk in~\Cref{theorem:conlinucbsk-regret}.

\begin{restatable}{lemma}{restatelemmasmoothedboundedxnorm}
    \label{lemma:smoothed-bounded-x-norm}
    For \conlinucbsk, for any \(a\in \mathcal{A}\), if \(t\geq T_0\triangleq \frac{8(1+\sqrt{d}R)^2}{b\lambda_{\mathcal{K}}}\log\ab(\frac{d}{\delta})\), with probability at least \(1-\delta\) for some \(\delta \in (0,1)\), $\|\bm{x}_a\|_{\bm{M}_t^{-1}} \leq \sqrt{\frac{2}{\lambda_{\mathcal{K}}bt}}$.
\end{restatable}

\begin{restatable}[Regret of \conlinucbsk]{theorem}{restateregretconlinucbsk}
    \label{theorem:conlinucbsk-regret}
    With probability at least \(1-\delta\) for some \(\delta \in (0,1)\), the regret upper bound of \conlinucbsk satisfies
    \begin{equation*}
    \begin{aligned}
    \label{eq:conlinucbsk-regret}
    & \R(T) \leq   \frac{8(1+\sqrt{d}R)^2\log(|\mathcal{K}|)}{c_1 \rho^2b}\log\ab(\frac{d}{\delta}) + 4\sqrt{\frac{2c_1 \rho^2T}{b\log(|\mathcal{K}|)}} \cdot \\
    &\  \ab(\sqrt{2\log{\ab(\frac{1}{\delta})}+d\log\ab(1+\frac{T + \ab(1 + \sqrt{d}R)bT}{\lambda d})} + \sqrt{\lambda}) \\
    & =  \mathcal{O}(\sqrt{dT\log(T)}+d),
    \end{aligned}
    \end{equation*}
    where \(R\) and \(\rho^2\) are constants  in~\Cref{def:smoothed-key-term-contexts}.
\end{restatable}

\subsection{Regret  Analysis of \conlinucbme Algorithm}
\label{subsec:regret-conlinusbme}

We begin with~\Cref{lemma:me-diff-estimate-true-reward}, which closely parallels~\Cref{lemma:smoothed-diff-estimate-true-reward}.

\begin{restatable}{lemma}{restatelemmamediffestimatetruereward}
    \label{lemma:me-diff-estimate-true-reward}
    Let \(\bm{\theta}_t\) be the estimated preference vector at round \(t\) and \(\bm{\theta}^{*}\) be the true preference vector. 
    Under Assumptions~\ref{assumption:normalized-vector}, ~\ref{assumption:subgaussian-noise} and~\ref{assumption:c0-for-key-term}, for \conlinucbme,
    at round \(t\), for any arm \(a \in \mathcal{A}\), with probability at least \(1-\delta\) (\(\delta \in (0,1)\)), we have
    \begin{align*}
    \ab|\bm{x}_a^{\top} \bm{\theta}_t - \bm{x}_a^{\top} \bm{\theta}^{*}| \leq \alpha_t \|\bm{x}_a\|_{\bm{M}_t^{-1}},
    \end{align*}
    where \(\alpha_t = \sqrt{2\log{(\frac{1}{\delta})}+d\log\ab(1+\frac{t + \alpha dt}{\lambda d c_0^2})} + \sqrt{\lambda}\), \(\alpha\) is an exploration control factor in Algorithm~\ref{alg:conlinucbme}, and \(c_0\) is a constant in~\Cref{assumption:c0-for-key-term}.
\end{restatable}

Since conversations are initiated adaptively in \conlinucbme, the number of conversations conducted up to each round \(t\) is not deterministic.
A key challenge to prove \Cref{lemma:me-diff-estimate-true-reward} is to bound this quantity.
Then, we present~\Cref{lemma:me-bounded-x-norm}, which bounds \(\|\bm{x}_a\|_{\bm{M}_t^{-1}}\).
\begin{restatable}{lemma}{restatelemmameboundedxnorm}
    \label{lemma:me-bounded-x-norm}
    For \conlinucbme, for any arm \(a\in \mathcal{A}\), with probability at least \(1-\delta\) for some \(\delta \in (0,1)\), 
    at round \(t\geq 2P\), we have
    $\|\bm{x}_a\|_{\bm{M}_t^{-1}} \leq \sqrt{\frac{2}{\alpha t}}$, where \(P\) is a fixed integer.
\end{restatable}

The proof of \Cref{lemma:me-bounded-x-norm} relies on establishing a lower bound on the minimum eigenvalue of \(\bm{M}_t\), i.e., \(\lambda_{\min}(\bm{M}_t) \geq \alpha t\), which involves a delicate analysis of covariance matrix eigenvalues.
The condition \(t\ge 2P\) is introduced to generalize all three checking functions.
Building on this, we derive the following theorem for \conlinucbme.

\begin{restatable}[Regret of \conlinucbme]{theorem}{restateregretconlinucbme}
    \label{theorem:conlinucbme-regret}
    With probability at least \(1-\delta\) for some \(\delta \in (0,1)\), the regret upper bound of \conlinucbme 
    satisfies
    \begin{equation*}
    \begin{aligned}
    \label{eq:conlinucbme-regret}
    \R(T) &\leq 4\sqrt{\frac{2T}{\alpha}} \ab(\sqrt{2\log{(\frac{1}{\delta})}+d\log\ab(1+\frac{T + \alpha dT}{\lambda d c_0^2})} + \sqrt{\lambda}) + 2P\\
    &= \mathcal{O}\ab(\sqrt{dT\log(T)}).
    \end{aligned}
    \end{equation*}
\end{restatable}
\begin{remark}
    Note that~\Cref{theorem:conlinucbme-regret} applies to all three uncertainty checking functions discussed in \conlinucbme algorithm, which underscores the generality of our methods.
\end{remark}

\conlinucbskme combines the advantages of both smoothed key term contexts and adaptive conversation techniques, ensuring efficient exploration while adaptively adjusting conversation frequency based on uncertainty.
As a result, we derive the following corollary.

\begin{restatable}{corollary}{restateregretcombined}
    \label{corollary:combined-regret}
    With probability at least \(1 - \delta\) for some \(\delta \in (0,1)\), the regret upper bound of \conlinucbskme satisfies \(\R(T) = \mathcal{O}(\sqrt{dT\log(T)})\).
\end{restatable}

\subsection{Lower Bound for Conversational Bandits}
  We establish a regret lower bound for conversational bandits with \emph{finite} and \emph{time-varying} arm sets.
  Our result is novel because the well-known lower bound \(\Omega(\sqrt{dT})\) by \citet{chu2011contextual} does not consider conversational information and thus cannot be directly applied to our setting.
  Additionally, the existing lower bound for federated conversational bandits~\citep{li2024fedconpe} is also inapplicable, as it assumes a \emph{fixed} arm set.
  The detailed proof is given in \Cref{sec:lower-bound}.
  
\begin{restatable}[Regret lower bound]{theorem}{restatelowerbound}\label{thm:lowerbound}
  For any policy that chooses at most one key term per time step, there exists an instance of the conversational bandit problem such that the expected regret is at least \(\Omega(\sqrt{dT})\).
  Furthermore, for any \(T=2^m\) with \(m \in [d]\), the regret is at least \(\Omega(\sqrt{dT\log(T)})\).
\end{restatable}

\subsection{Discussion on Optimality}
\label{subsec:diss}
To the best of our knowledge, we are the first to propose algorithms for conversational contextual bandits that achieve the \emph{optimal} regret bound of order \(\widetilde{\mathcal{O}}(\sqrt{dT})\).
We summarize the regret bounds of our proposed algorithms and related algorithms in \Cref{table:regret-comparison} and discuss the theoretical improvements over existing methods.

\begin{table}[htb]
\centering
\setlength{\extrarowheight}{0pt}
\addtolength{\extrarowheight}{\aboverulesep}
\addtolength{\extrarowheight}{\belowrulesep}
\setlength{\aboverulesep}{0pt}
\setlength{\belowrulesep}{0pt}
\caption{Comparison of theoretical regret bounds.}
\label{table:regret-comparison}
\resizebox{\linewidth}{!}{%
\begin{threeparttable}[b]
\begin{tabular}{lll} 
\toprule
\textbf{Algorithm}                                                                        & \textbf{Conversational} & \textbf{Regret}                    \\ 
\midrule
LinUCB~\cite{abbasi2011improved}                                                          & \xmark                  & $\mathcal{O}(d\sqrt{T}\log{T})$    \\
ConUCB~\cite{zhang2020conversational}, ConLinUCB-MCR~\cite{wang2023efficient}             & \cmark                  & $\mathcal{O}(d\sqrt{T}\log{T})$    \\
ConLinUCB-BS~\cite{wang2023efficient}                                                     & \cmark                  & At least $\mathcal{O}(d\sqrt{T\log{T}})$\tnote{*}    \\
\rowcolor[rgb]{0.753,0.753,0.753} \conlinucbsk (Ours, \Cref{theorem:conlinucbsk-regret})  & \cmark                  & $\mathcal{O}(\sqrt{dT\log{T}}+d)$  \\
\rowcolor[rgb]{0.753,0.753,0.753} \conlinucbme (Ours, \Cref{theorem:conlinucbme-regret})  & \cmark                  & $\mathcal{O}(\sqrt{dT\log{T}})$    \\
\rowcolor[rgb]{0.753,0.753,0.753} \conlinucbskme (Ours, \Cref{corollary:combined-regret}) & \cmark                  & $\mathcal{O}(\sqrt{dT\log{T}})$    \\
\bottomrule
\end{tabular}
  \begin{tablenotes}
  \item [*] The original paper claims a regret of $\mathcal{O}(\sqrt{dT\log{T}})$ but its analysis is flawed.
  \end{tablenotes}
\end{threeparttable}
}
\end{table}

The regret upper bound of LinUCB~\cite{abbasi2011improved} is \(\mathcal{O}(d\sqrt{T}\log{T})\), which serves as a standard benchmark in contextual linear bandits.
The first algorithm for conversational bandits, ConUCB~\cite{zhang2020conversational}, offers the same regret upper bound as LinUCB, indicating that it does not offer a substantial theoretical improvement over the non-conversational algorithms.
Since then, improving regret through conversational interactions has remained an open problem in the field.
Under the assumption that the key term set \(\mathcal{K}\) spans \(\mathbb{R}^d\), ConLinUCB-BS~\cite{wang2023efficient} achieves a regret upper bound of \(\mathcal{O}(\frac{1}{\sqrt{\lambda_{\mathcal{B}}}}\sqrt{dT\log{T}})\), where \(\lambda_{\mathcal{B}} \coloneq \lambda_{\min} \ab(\E_{k\in\text{unif}(\mathcal{B})}\ab[\tilde{\bm{x}}_k\tilde{\bm{x}}_k^{\top}]) \) and \(\mathcal{B}\) is the \emph{barycentric spanner} of \(\mathcal{K}\).
The authors assume \(\lambda_{\mathcal{B}}\) is a constant, leading to a regret bound of \(\mathcal{O}(\sqrt{dT\log{T}})\).
However, this assumption is incorrect as \(\lambda_{\mathcal{B}}\) depends on the dimension \(d\) and is not a constant.
Specifically, denoting \(\bm{X} \coloneq \E_{k\in\text{unif}(\mathcal{B})}\ab[\tilde{\bm{x}}_k\tilde{\bm{x}}_k^{\top}]\) and \(\{\lambda_i\}_{i=1}^{d}\) as its eigenvalues, we use the fact that \(\|\tilde{\bm{x}}_k\|=1\) and obtain \(\Tr{\ab(\bm{X})} = \E_{k\in\text{unif}(\mathcal{B})}\ab[\Tr\ab(\tilde{\bm{x}}_k\tilde{\bm{x}}_k^{\top})] = 1\), thus \(\lambda_{\mathcal{B}} \leq \frac{\sum_{i=1}^{d}\lambda_i}{d} = \frac{\Tr{\ab(\bm{X})}}{d} = \frac{1}{d}\).
Consequently, by plugging this result back into the regret expression, the regret bound of ConLinUCB-BS cannot be better than \(\mathcal{O}(d\sqrt{T\log{T}})\).
These previous attempts underscore the significance of our work.
In contrast, with the smoothed key term context technique and with the adaptive conversation technique, our algorithms achieve a better regret bound of \(\mathcal{O}(\sqrt{dT\log{T}}+d)\) and \(\mathcal{O}(\sqrt{dT\log{T}})\), respectively.
These improvements successfully match the lower bound (\Cref{thm:lowerbound}) up to logarithmic factors in their dependence on the time horizon \(T\).

\section{Evaluation}
\label{sec:evaluation}
In this section, we evaluate the performance of our algorithms on both synthetic and real-world datasets.
All the experiments were conducted on a machine equipped with a 3.70 GHz Intel Xeon E5-1630 v4 CPU and 32GB RAM.

\subsection{Experiment Setups}
\subsubsection{Datasets}
Consistent with existing studies, we generate a synthetic dataset and use three real-world datasets: MovieLens-25M~\cite{harper-2015-movielens}, Last.fm~\cite{cantador-2011-second-workshop}, and Yelp\footnote{\url{https://www.yelp.com/dataset}}.

For the synthetic dataset, we set the dimension \(d=50\), the number of users \(N=200\), the number of arms \(|\mathcal{A}|=5,000\), and the number of key terms \(|\mathcal{K}|=1,000\).
We generate it following~\citet{zhang2020conversational}.
First, for each key term \(k \in \mathcal{K}\), we sample a pseudo feature vector \(\dot{\bm{x}}_k\) with each dimension drawn from a uniform distribution \(\mathcal{U}(-1,1)\). For each arm \(i \in \mathcal{A}\), we randomly select an integer \(n_i \in \{1,2, \dots, 5\}\) and uniformly sample a subset of key terms \(\mathcal{K}_i \subset \mathcal{K}\) with \(|\mathcal{K}_i|=n_i\). The weight is defined as \(w_{i,k}=1/n_i\) for each \(k \in \mathcal{K}_i\). For each arm \(i\), the feature vector \(\bm{x}_i\) is drawn from a multivariate Gaussian \(\mathcal{N}(\sum_{j \in \mathcal{K}_i} \dot{\bm{x}}_j/n_i, \bm{I})\). The feature vector for each key term \(k\), denoted by \(\tilde{\bm{x}}_k\), is computed as \(\tilde{\bm{x}}_k = \sum_{i \in \mathcal{A}} \frac{w_{i,k}}{\sum_{j \in \mathcal{A}}w_{j,k}}\bm{x}_i\). Finally, each user's preference vector \(\bm{\theta}_u \in \mathbb{R}^d\) is generated by sampling each dimension from \(\mathcal{U}(-1,1)\) and normalizing it to unit length.

For the real-world datasets, we regard movies\slash artists\slash businesses as arms.
To exclude unrepresentative or insufficiently informative data (such as users who have not submitted any reviews or movies with only a few reviews), we extract a subset of \(|\mathcal{A}|=5,000\) arms with the highest number of user-assigned ratings/tags, and a subset of \(N=200\) users who have assigned the most ratings/tags.
Key terms are identified by using the associated movie genres, business categories, or tag IDs in the MovieLens, Yelp, and Last.fm datasets, respectively.
For example, each movie is associated with a list of genres, such as ``action'' or ``comedy'', and each business (e.g., restaurant) is categorized by terms such as ``Mexican'' or ``Burgers''.
Using the data extracted above, we create a \emph{feedback matrix} \(\bm{R}\) of size \({N \times |\mathcal{A}|}\), where each element \(\bm{R}_{i,j}\) represents the user \(i\)'s feedback to arm \(j\).
We assume that the user's feedback is binary.
For the MovieLens and Yelp datasets, a user's feedback for a movie/business is 1 if the user's rating is higher than 3; otherwise, the feedback is 0.
For the Last.fm dataset, a user's feedback for an artist is 1 if the user assigns a tag to the artist.
 Next, we generate the feature vectors for arms \(\bm{x}_i\) and the preference vectors for users \(\bm{\theta}_u\).
Following existing works, we decompose the feedback matrix \(\bm{R}\) using truncated Singular Value Decomposition (SVD) as \(\bm{R} \approx \bm{\Theta} \bm{S} \bm{A}^{\top}\), where \(\bm{\Theta} \in \RR^{N \times d}\) and \(\bm{A} \in \RR^{|\mathcal{A}| \times d}\) contain the top-\(d\) left and right singular vectors, and \(\bm{S} \in \RR^{d\times d}\) is a diagonal matrix with the corresponding top-\(d\) singular values.
Then each \(\bm{\theta}_u^\top\) corresponds to the \(u\)-th row of \(\bm{\Theta}\bm{S}\) for all \(u \in [N]\), and each \(\bm{x}_i^\top\) corresponds to the \(i\)-th row of \(\bm{A}\) for all \(i \in \mathcal{A}\).
The feature vectors for key terms are generated similarly to those in the synthetic dataset, by assigning equal weights for all key terms corresponding to each arm.

\subsubsection{Baseline Algorithms}
We select the following baseline algorithms from existing studies:
(1) LinUCB~\cite{abbasi2011improved}: The standard linear contextual bandit algorithm, which does not consider the conversational setting and only has arm-level feedback.
(2) Arm-Con~\cite{christakopoulou2016towards}: An extension of LinUCB that initiates conversations directly from arm sets.
(3) ConUCB~\cite{zhang2020conversational}: The first algorithm proposed for conversational contextual bandits that queries key terms when conversations are allowed.
(4) ConLinUCB~\cite{wang2023efficient}: It consists of three algorithms with different key term selection strategies.
ConLinUCB-BS computes the \emph{barycentric spanner} of key terms as an exploration basis.
ConLinUCB-MCR selects key terms with the largest confidence radius.
ConLinUCB-UCB chooses key terms with the largest upper confidence bounds.
Since ConLinUCB-BS and ConLinUCB-MCR demonstrate superior performance, we focus our comparisons on these two variants.

\subsection{Evaluation Results}
\label{sec:evaluation-results}
\subsubsection{Cumulative Regret}
First, we compare our algorithms against all baseline algorithms in terms of cumulative regret over \(T=6,000\) rounds.
In each round, we randomly select \(|\mathcal{A}|=200\) arms from each dataset.
For the baseline algorithms, we adopt the conversation frequency function \(b(t) = 5\lfloor \log(T)\rfloor\), as specified in their original papers.
We present the results for all three checking functions ``Continuous'', ``Fixed Interval'', and "Exponential Phase'', for both \conlinucbme and \conlinucbskme.
For the ``Fixed Interval'' function, \texttt{UncertaintyChecking} is triggered every 100 rounds, whereas for the "Exponential Phase'' it is triggered whenever \(t\) is a power of 2. 
For \conlinucbsk, both the perturbation level \(\rho^2\) and the truncation limit \(R\) are set to 1.
The results are averaged over 20 trials, and the resulting confidence intervals are included in the figures.
Under the ``Continuous Checking'' function, as shown in \Cref{fig:regret}, our three algorithms consistently achieve the best performance (lowest regret) with an improvement of over 14.6\% compared to the best baseline.
Similar performance trends hold under the other two checking functions, as illustrated in \Cref{fig:regret-interval,fig:regret-exponential}.
These results confirm the validity of our theoretical advancements.

\begin{figure}[htb]
    \centering
    \includegraphics[width=\linewidth]{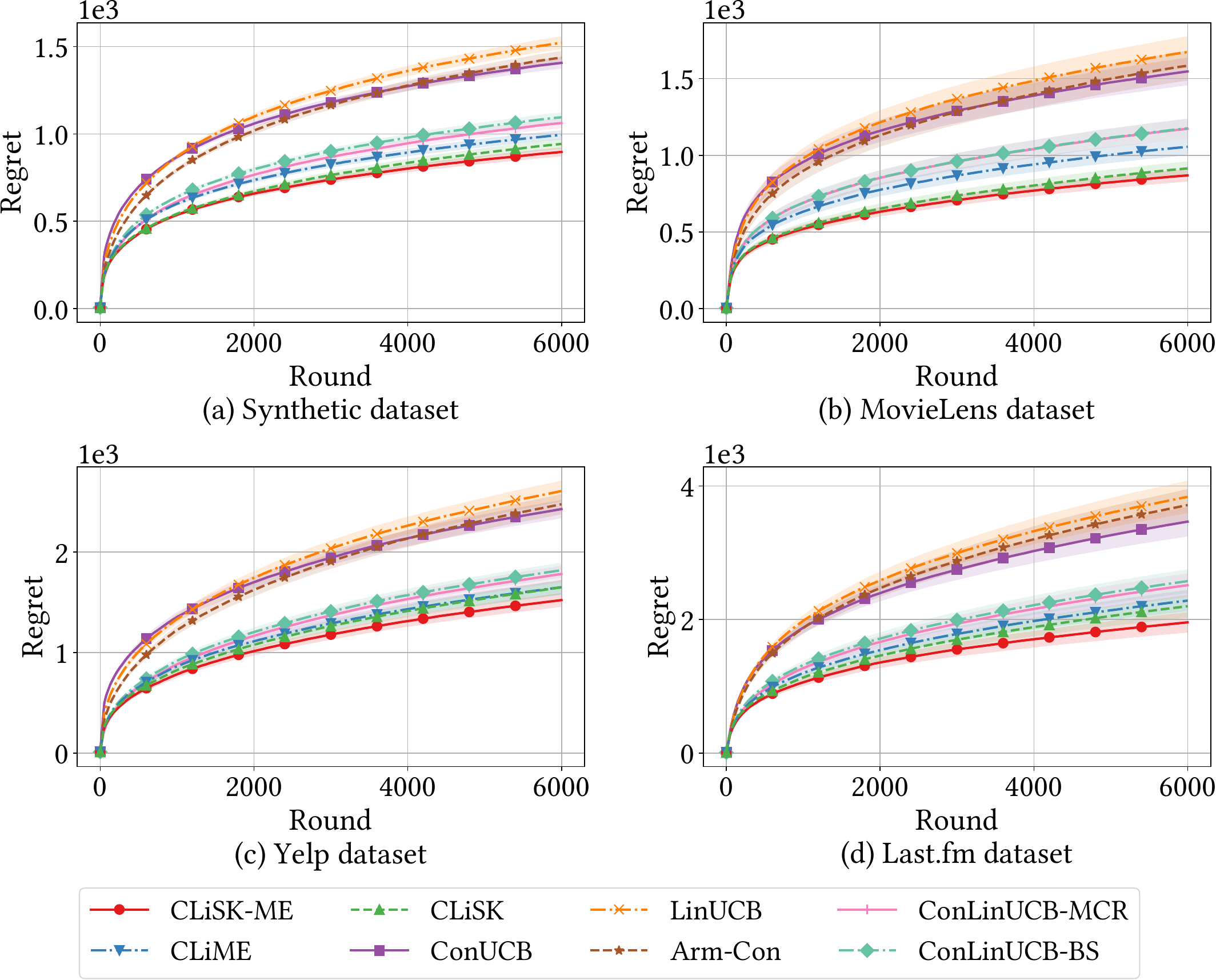}
    \caption{Comparison of cumulative regret where \conlinucbme and \conlinucbskme use the ``Continuous Checking'' function.}
    \label{fig:regret}
\end{figure}

\begin{figure}[htb]
    \centering
    \includegraphics[width=\linewidth]{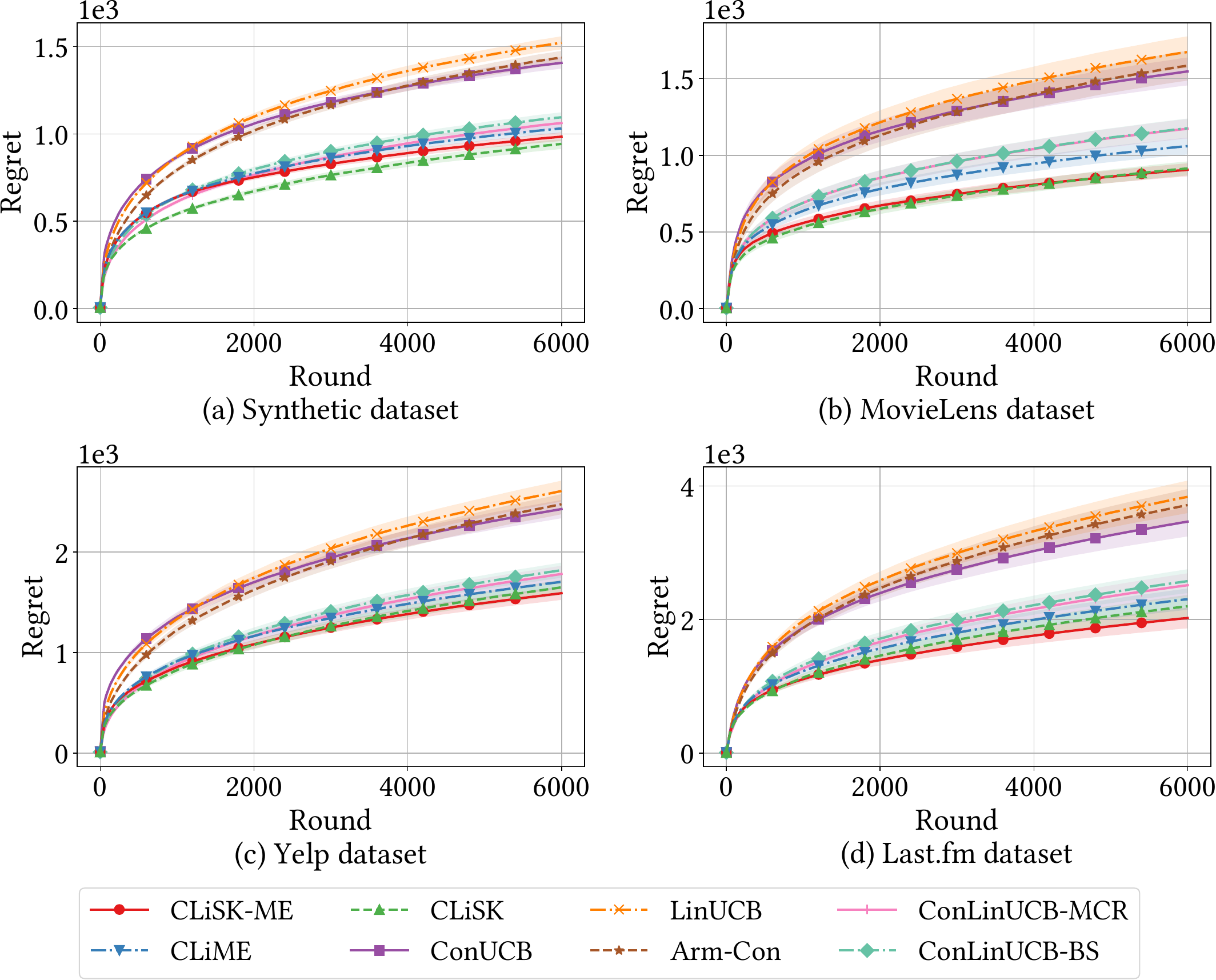}
    \caption{Comparison of cumulative regret where \conlinucbme and \conlinucbskme use the ``Fixed Interval'' function.}
    \label{fig:regret-interval}
\end{figure}

\begin{figure}[htb]
    \centering
    \includegraphics[width=\linewidth]{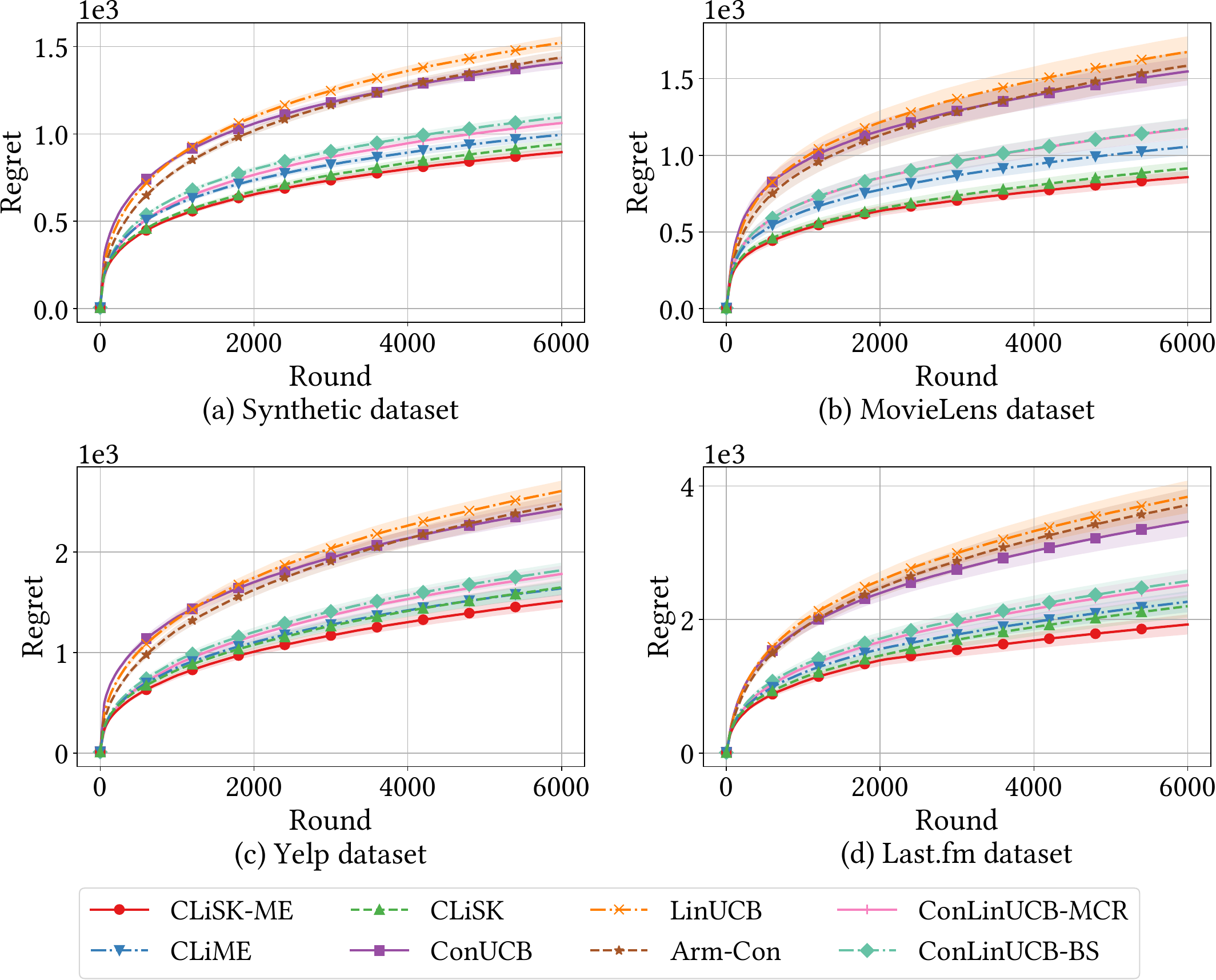}
    \caption{Comparison of cumulative regret where \conlinucbme and \conlinucbskme use the ``Exponential Phase'' function.}
    \label{fig:regret-exponential}
\end{figure}

\subsubsection{Precision of Estimated Preference Vectors}
To assess how accurately each algorithm learns the user's preferences over time, we measure the average distance between the estimated vector \(\widehat{\vec{\theta}}_t\) and the ground truth \(\vec{\theta}^{*}\) for all algorithms over 1000 rounds.
We present the results for the ``Continuous Checking'' function of \conlinucbme and \conlinucbskme, with results for other functions provided in \Cref{app:extra-experiments}.
As shown in \Cref{fig:theta-diff}, all algorithms exhibit a decreasing estimation error over time.
However, our three algorithms consistently achieve the lowest estimation error in all datasets.
This is because they leverage our novel conversational mechanism to gather more informative feedback, significantly accelerating the reduction of estimation error.
As a result, our algorithms estimate the user's preference vector more quickly and accurately than the baseline methods.

\begin{figure}[htb]
    \centering
    \includegraphics[width=\linewidth]{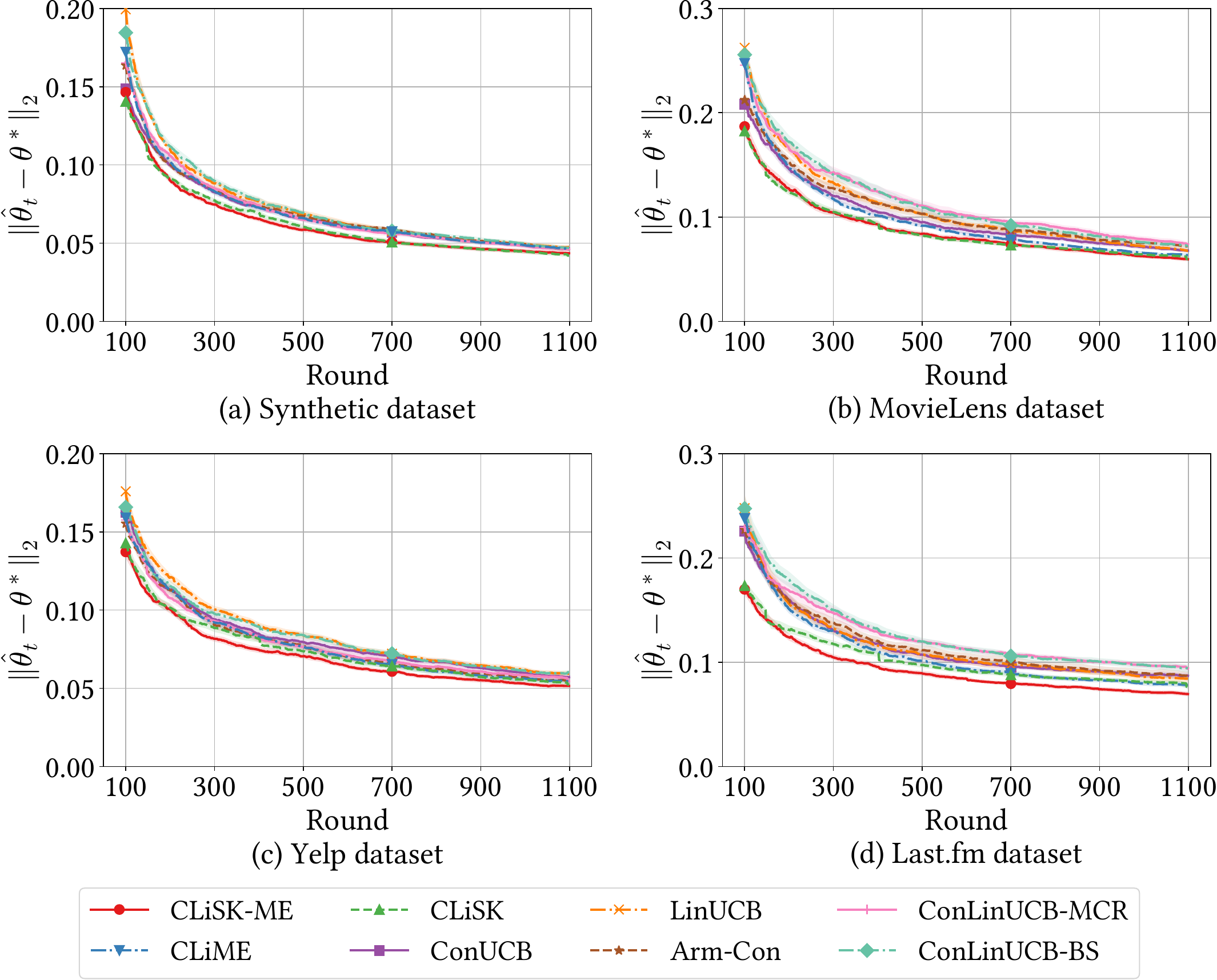}
    \caption{Comparison of estimation precision where \conlinucbme and \conlinucbskme use the ``Continuous Checking'' function.}
    \label{fig:theta-diff}
\end{figure}

\subsubsection{Number of Conversations}
Next, we evaluate the number of conversations initiated by \conlinucbme. Since \conlinucbsk and all baseline algorithms initiate conversations based on a deterministic function \(b(t)\), their results are consistent across all datasets. Therefore, we plots the scenarios for \(b(t) = 5\lfloor \log(t) \rfloor\) and \(b(t) = \lfloor t/50 \rfloor\) as in prior studies.
It is also important to note that although some existing studies employ a logarithmic \(b(t)\) in their experiments, their theoretical results require a linear \(b(t)\) to hold.
In contrast to the baselines, our algorithm \conlinucbme adaptively initiates conversations depending on the current uncertainty of user preferences, providing greater flexibility and enhancing the user experience. We plot the number of conversations initiated by \conlinucbme with different uncertainty checking functions across 4 datasets. As shown in \Cref{fig:keyterm-pulling-time}, the number of conversations increases only logarithmically with the number of rounds.

\begin{figure}[htb]
    \centering
    \includegraphics[width=\linewidth]{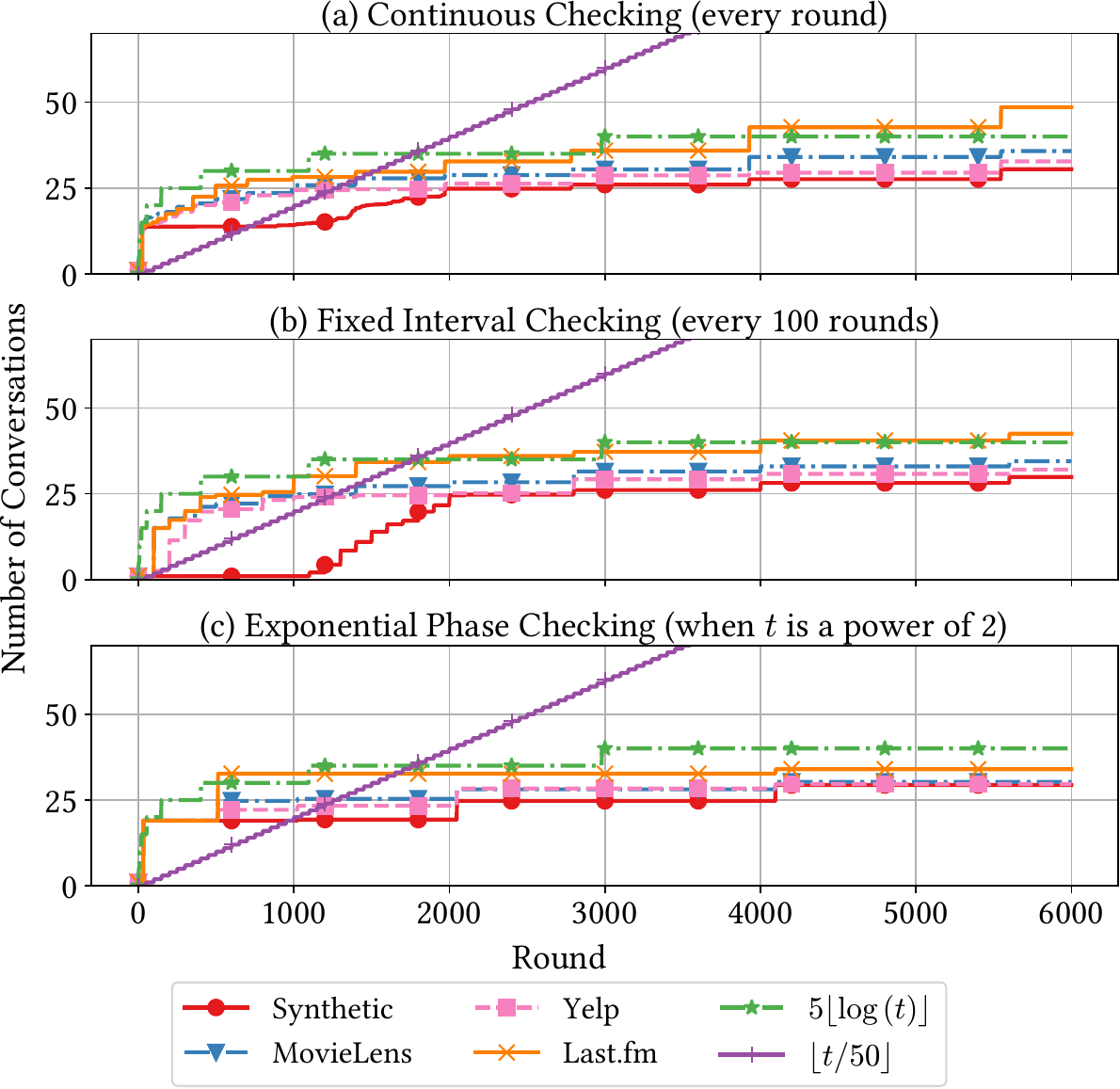}
    \caption{Number of conversations initiated by deterministic approaches and our adaptive approach \conlinucbme with different uncertainty checking functions.}
    \label{fig:keyterm-pulling-time}
\end{figure}

\subsubsection{Running Time}
To evaluate the computational efficiency, we compare the running times of our algorithms with other conversational methods using the MovieLens dataset across \(T=6,000\) rounds.
We separately report the total running times, as well as the times for picking arms and key terms.
The results are averaged over 20 runs.
As shown in \Cref{table:running-time-comparison},
our three algorithms show substantial improvements compared to ConUCB and exhibit performance comparable to the ConLinUCB family of algorithms.
For \conlinucbme and \conlinucbskme,  while matrix operations and eigenvalue computation introduce slight overhead, the algorithms remain efficient, particularly with interval and exponential checking strategies.

\begin{table}[htb]
\centering
\caption{Comparison of Running Times for Conversational Bandit Algorithms Using the Movielens dataset.}
\label{table:running-time-comparison}
\begin{tabular}{llrrr} 
\toprule
\multicolumn{2}{c}{\multirow{2}{*}{\textbf{Algorithms}}} & \multicolumn{3}{c}{\textbf{Running Time (s)}}  \\ 
\cmidrule(lr){3-5}
\multicolumn{2}{c}{}                                     & \multicolumn{1}{l}{Key terms} & Arms  & Total  \\ 
\midrule
\multirow{3}{*}{\conlinucbskme} & Continuous             & 1.169                         & 3.443 & 4.651  \\
                                & Interval               & 0.332                         & 3.361 & 3.723  \\
                                & Exponential            & 0.352                         & 3.344 & 3.724  \\ 
\cmidrule(lr){1-5}
\multirow{3}{*}{\conlinucbme}   & Continuous             & 0.803                         & 3.371 & 4.205  \\
                                & Interval               & 0.021                         & 3.341 & 3.390  \\
                                & Exponential            & 0.014                         & 3.334 & 3.375  \\ 
\cmidrule(lr){1-5}
\conlinucbsk                    &                        & 0.490                         & 3.339 & 3.857  \\ 
\cmidrule(lr){1-5}
ConUCB                          &                        & 0.011                         & 8.362 & 8.403  \\ 
\cmidrule(lr){1-5}
\multirow{3}{*}{ConLinUCB}       & UCB                    & 0.009                         & 3.354 & 3.392  \\
                                & MCR                    & 0.007                         & 3.337 & 3.371  \\
                                & BS                     & 0.006                         & 3.334 & 3.366  \\
\bottomrule
\end{tabular}
\end{table}

\subsubsection{Ablation Study}
We conduct an ablation study evaluating the effect of the truncation limit \(R\).
Specifically, we analyze how different values of  \(R\) affect algorithm performance by comparing the cumulative regrets at round 6,000 across all datasets, as shown in \Cref{fig:regret-vs-R}.
The results indicate that increasing \(R\) from 0.1 to 3.1 leads to a decrease in regret, with performance stabilizing when \(R>2\).
For the perturbation level \(\rho^2\), we observe that varying it from 0.1 to 3 results in no significant change in regret. 
Therefore, we do not include a separate figure for this parameter.

\begin{figure}[htb]
    \centering
    \includegraphics[width=\linewidth]{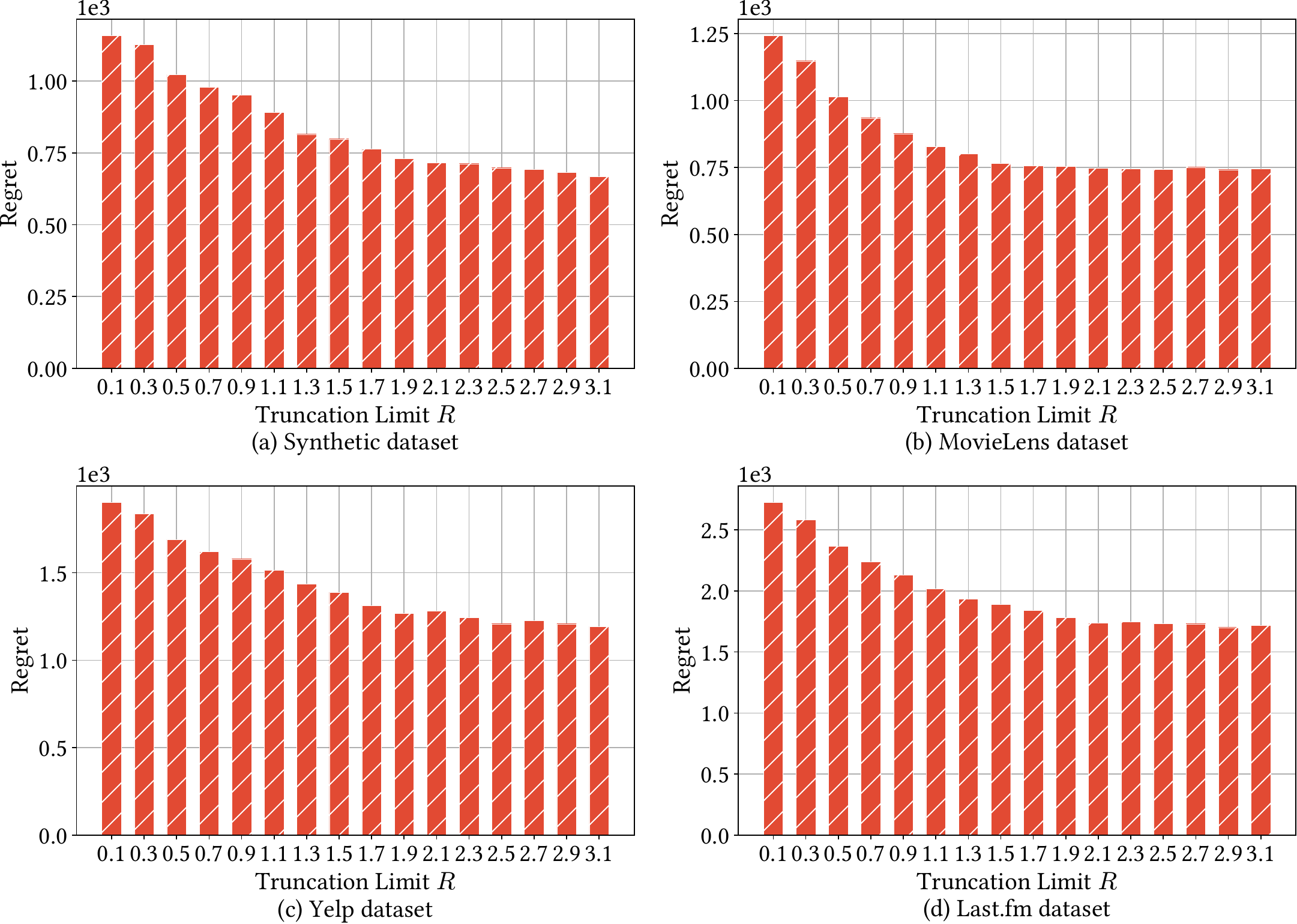}
    \caption{Effect of the truncation limit \(R\).}
    \label{fig:regret-vs-R}
\end{figure}

\section{Related Work}
\label{sec:related}
Our research is closely aligned with studies on conversational contextual bandits, particularly focusing on the problem of key term selection within this framework.

Contextual bandits serve as a fundamental framework for online sequential decision-making problems, covering applications like recommender systems~\cite{li2010contextual,chu2011contextual} and computer networking~\cite{gai2012combinatorial}.
Contextual bandit algorithms aim to maximize the cumulative reward in the long run while making the trade-off between exploitation and exploration.
Prominent algorithms include LinUCB~\cite{abbasi2011improved}  and Thompson Sampling (TS)~\cite{agrawal2012analysis}.

To address the cold start problem, conversational recommender systems (CRSs)~\cite{christakopoulou2016towards, sun2018conversational, zhang2018towards} are proposed to engage users in conversations to learn their preferences more effectively.
\citet{zhang2020conversational} extend the standard contextual bandits to model conversational interactions, and the pioneering ConUCB algorithm with a regret upper bound \(\mathcal{O}(d\sqrt{T}\log{T})\).
Following the foundational work of~\citet{zhang2020conversational}, a branch of research has advanced this field.
\citet{li2021seamlessly} design the first TS-type algorithm ConTS.
\citet{wu2021clustering} propose a clustering-based algorithm to automatically generate key terms.
\citet{zuo2022hierarchical} propose Hier-UCB and Hier-LinUCB, leveraging the hierarchical structures between key terms and items.
\citet{xie2021comparison} introduce a comparison-based conversation framework and propose RelativeConUCB.
\citet{zhao2022knowledge} integrate knowledge graphs into conversational bandits.
\citet{li2024fedconpe} investigate federated conversational bandits.
\citet{dai2024conversational,dai2024online} study the conversational bandits with misspecified/corrupted models.
To enhance learning efficiency,
\citet{dai-2025-multi-agent} consider multi-agent LLM response identification with a fixed arm set.
\citet{wang2023efficient} and~\citet{yang2024conversational} investigate the key term selection strategies and propose the ConLinUCB-BS and ConDuel algorithms, respectively. Both algorithms uniformly select key terms from the barycentric spanner of the key term set.

The smoothed analysis for contextual bandits has been widely studied recently~\cite{kannan2018smoothed, sivakumar2020structured, raghavan2023greedy, raghavan2018externalities,li-2025-towards,li-2025-demystifying}.
The smoothed setting bridges i.i.d. distributional and adversarial contexts.
\citet{kannan2018smoothed} first introduce the smoothed analysis for linear contextual bandits, showing that small perturbations can lead to sublinear regret with a greedy algorithm.
\citet{raghavan2018externalities} and \citet{raghavan2023greedy} show that the greedy algorithm achieves the best possible Bayesian regret in this setting.
\citet{sivakumar2020structured} extend the smoothed analysis to structured linear bandits.
Building on these insights, we apply the smoothed key term contexts in conversational contextual bandits.

\section{Conclusion} 
\label{sec:conclusion}

In this paper, we studied key term selection strategies for conversational contextual bandits and introduced three novel algorithms: \conlinucbsk, \conlinucbme, and \conlinucbskme.
\conlinucbsk leverages smoothed key term contexts to enhance exploration, while \conlinucbme adaptively initiates conversations with key terms that minimize uncertainty in the feature space.  \conlinucbskme integrates both techniques, further improving learning efficiency.
We proved that all three algorithms achieve tighter regret bounds than prior studies.
Extensive evaluations showed that our algorithms outperform other conversational bandit algorithms.

\begin{acks}
The work of John C.S. Lui was supported in part by the RGC GRF-14202923.
\end{acks}

\newpage

\bibliographystyle{ACM-Reference-Format}
\balance
\bibliography{references}

\appendix
\allowdisplaybreaks %

\section{Appendix}

\subsection{\conlinucbskme Algorithm}
\label{app:conlincubskme}
In this section, we present the details of the \conlinucbskme algorithm (Algorithm~\ref{alg:conlinucbskme}), which integrates the smoothed key term contexts and the adaptive conversation technique.
\begin{algorithm}[htb]
    \DontPrintSemicolon
    \SetKwComment{Comment}{$\triangleright$\ }{}
    \SetKwInput{KwInit}{Initialization}
    \KwIn{\(\mathcal{A}\), \(\mathcal{K}\), \(b(t)\), \(\lambda\), \(\{\alpha_t\}_{t>0}\)}
    \KwInit{\(\bm{M}_1 = \lambda \bm{I}_d\), \(\bm{b}_1 = \bm{0}_d\)}
    \For{\(t = 1, \dots, T\) }{
    
            \If{\UncertaintyChecking{\(t\)}}{ \label{line-skme:check-uncertainty}
            Diagonalize \(\bm{M}_t = \sum_{i=1}^d \lambda_{\bm{v}_i} \bm{v}_i{\bm{v}_i}^{\top}\)\; \label{line-skme:diagonalize}
                \ForEach{\(\lambda_{\bm{v}_i} < \alpha t\)}{ \label{line-skme:check-lambda}
                \(n_{\bm{v}_i} = \lceil(\alpha t -\lambda_{\bm{v}_i})/c_0^2\rceil\)\; \label{line-skme:compute-n}
                \For{\(n_{\bm{v}_i} > 0\) }{
                Smooth the key term contexts to get \(\{\doubletilde{\bm{x}}_{k}\}_{k\in \mathcal{K}}\)\;
                \label{line-skme:smooth-key-term}
                \(k = \argmax_{k\in \mathcal{K}}|\doubletilde{\bm{x}}_k^{\top}\bm{v}_i|\)\;  \label{line-skme:select-k}
                Receive the key term-level feedback \(\tilde{r}_{k,t}\)\; \label{line-skme:receive-key-term-feedback}
                \(\bm{M}_t = \bm{M}_{t} + \doubletilde{\bm{x}}_{k,t}\doubletilde{\bm{x}}_{k,t}^{\top}\)\; \label{line-skme:update-M-key-term}
                \(\bm{b}_t = \bm{b}_{t} + \tilde{r}_{k,t}\doubletilde{\bm{x}}_{k,t}\)\; \label{line-skme:update-b-key-term}
                \(n_{\bm{v}_i} = n_{\bm{v}_i} -1\)\; \label{line_skme:update-n}
                }
                }
            }
        \(\bm{\theta}_t = \bm{M}_t^{-1}\bm{b}_t\)\; \label{line-skme:update-theta}
        Select \(a_t = \argmax_{a \in \mathcal{A}_t} \bm{x}_a^{\top} \bm{\theta}_t  + \alpha_t \|\bm{x}_a\|_{M_t^{-1}}\)\; \label{line-skme:select-arm}
        Ask the user's preference for arm \(a_t\)\; \label{line-skme:ask-arm}
        Observe the reward \(r_{a_t,t}\)\; \label{line-skme:observe-reward}
        \(\bm{M}_{t+1} = \bm{M}_t + \bm{x}_{a_t}\bm{x}_{a_t}^{\top}\)\; \label{line-skme:update-M-arm}
        \(\bm{b}_{t+1} = \bm{b}_t + r_{a_t,t}\bm{x}_{a_t}\)\; \label{line-skme:update-b-arm}
    }
    \caption{\conlinucbskme} \label{alg:conlinucbskme}
  \end{algorithm}

\subsection{Supplementary Experiment Results}
\label{app:extra-experiments}
We compare the estimation precision for the ``Fixed Interval'' and ``Exponential Phase'' uncertainty checking functions of \conlinucbme in \Cref{fig:theta-diff-interval,fig:theta-diff-exponential}. In the former, \texttt{UncertaintyChecking} is triggered every 100 rounds while in the latter it is triggered when \(t\) is a power of 2. Combined with the results presented in the evaluation results section, the experiments demonstrate that our algorithms consistently outperform the baselines.

\begin{figure}[htb]
    \centering
    \includegraphics[width=\linewidth]{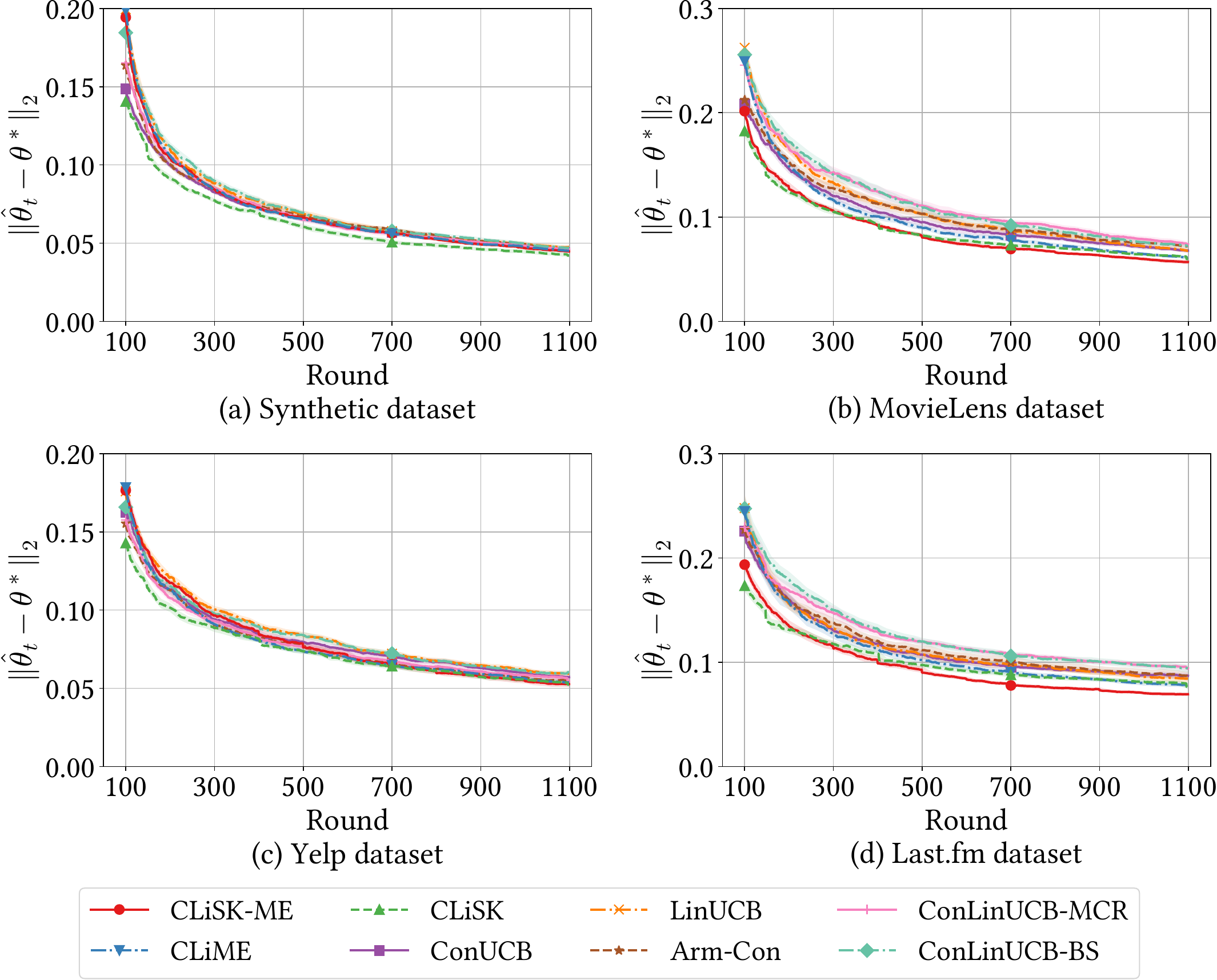}
    \caption{Comparison of estimation precision  where \conlinucbme and \conlinucbskme use the ``Fixed Interval'' function.}
    \label{fig:theta-diff-interval}
\end{figure}

\begin{figure}[htb]
    \centering
    \includegraphics[width=\linewidth]{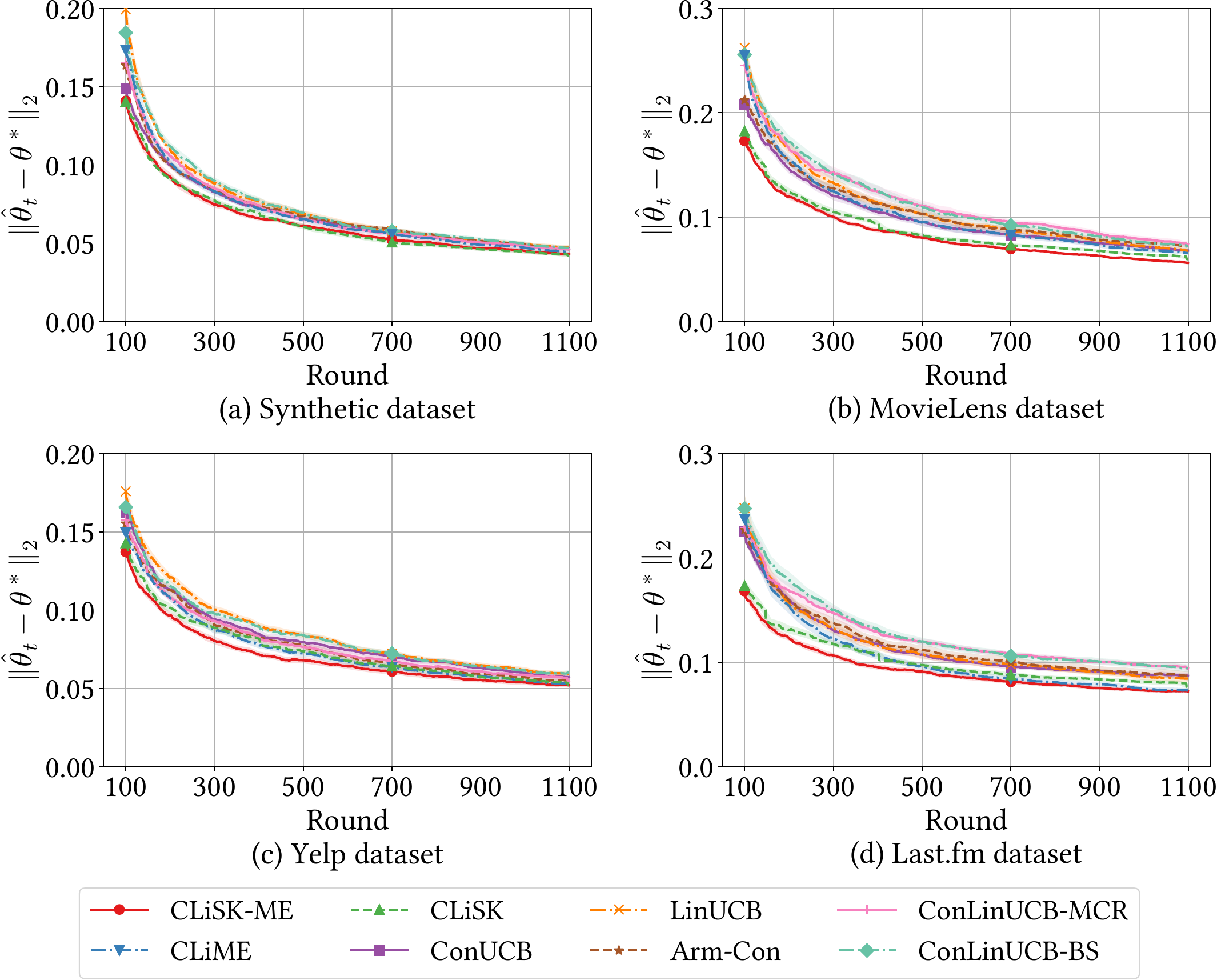}
    \caption{Comparison of estimation precision  where \conlinucbme and \conlinucbskme use the ``Exponential Phase'' function.}
    \label{fig:theta-diff-exponential}
\end{figure}

\subsection{Proof of~\Cref{lemma:smoothed-diff-estimate-true-reward}}
\restatelemmasmootheddiffestimatetruereward*

\begin{proof}
For any arm \(a\in \mathcal{A}\), from the definition of \(\bm{M}_t\) and \(\bm{b}_t\), and \(\bm{\theta}_t = \bm{M}_t^{-1} \bm{b}_t\), we have 
 \begin{equation*}
    \begin{aligned}
        & \bm{x}_a^{\top}\ab(\bm{\theta}_t  - \bm{\theta}^*) = \bm{x}_a^{\top}\ab(\bm{M}_t^{-1} \bm{b}_t - \bm{\theta}^* )\\
        =& \bm{x}_a^{\top}\ab(\bm{M}_t^{-1} \ab(\sum_{s=1}^{t-1} r_{a_s,s} \bm{x}_{a_s} + \sum_{s=1}^{t}\sum_{k \in \mathcal{K}_{s}}\tilde{r}_{k,s} \doubletilde{\bm{x}}_{k}) - \bm{\theta}^* ) \\
        =& \bm{x}_a^{\top}\ab(\bm{M}_t^{-1} \ab(\sum_{s=1}^{t-1} \bm{x}_{a_s} \ab(\bm{x}_{a_s}^{\top} \bm{\theta}^* + \eta_s) + \sum_{s=1}^{t}\sum_{k \in \mathcal{K}_{s}}\doubletilde{\bm{x}}_{k} \ab(\doubletilde{\bm{x}}_{k}^{\top} \bm{\theta}^* + \tilde{\eta}_s)) - \bm{\theta}^* ) \\
        =& \bm{x}_a^{\top}\ab(\bm{M}_t^{-1} \ab(\sum_{s=1}^{t-1} \bm{x}_{a_s}\bm{x}_{a_s}^{\top}  + \sum_{s=1}^{t}\sum_{k \in \mathcal{K}_{s}}\doubletilde{\bm{x}}_{k}\doubletilde{\bm{x}}_{k}^{\top} + \lambda \bm{I}_d - \lambda \bm{I}_d)\bm{\theta}^* - \bm{\theta}^* )\\
        &+ \bm{x}_a^{\top}\ab(\bm{M}_t^{-1} \ab(\sum_{s=1}^{t-1} \bm{x}_{a_s} \eta_s + \sum_{s=1}^{t}\sum_{k \in \mathcal{K}_{s}}\doubletilde{\bm{x}}_{k} \tilde{\eta}_s)) \\
        =& \lambda\bm{x}_a^{\top}\bm{M}_t^{-1} \bm{\theta}^* + \bm{x}_a^{\top}\ab(\bm{M}_t^{-1} \ab(\sum_{s=1}^{t-1} \bm{x}_{a_s} \eta_s + \sum_{s=1}^{t}\sum_{k \in \mathcal{K}_{s}}\doubletilde{\bm{x}}_{k} \tilde{\eta}_s)). 
    \end{aligned}
\end{equation*}

By the Cauchy-Schwarz inequality, we have
\begin{align*}
    \ab|\bm{x}_a^{\top}\ab(\bm{\theta}_t  - \bm{\theta}^*)|
    \leq&\ \lambda \ab\|\bm{x}_a\|_{\bm{M}_t^{-1}} \|\bm{\theta}^*\|_{\bm{M}_t^{-1}}\\
    &+ \ab\|\bm{x}_a\|_{\bm{M}_t^{-1}} \ab\| \sum_{s=1}^{t-1} \bm{x}_{a_s} \eta_s + \sum_{s=1}^{t}\sum_{k \in \mathcal{K}_{s}}\doubletilde{\bm{x}}_{k} \tilde{\eta}_s\|_{\bm{M}_t^{-1}}. \numberthis\label{eq:diff-estimate-true-reward-1}
\end{align*}

For the first term, by the fact that \(\lambda_{\min}(\bm{M}_t) \geq \lambda\), and by the property of the Rayleigh quotient, we have
\begin{align*}
    \frac{\ab\|\bm{\theta}^*\|_{\bm{M}_t^{-1}}^{2}}{\ab\|\bm{\theta}^*\|_2^2} = \frac{\bm{\theta}^{*\top}\bm{M}_t^{-1}\bm{\theta}^*}{\bm{\theta}^{*\top}\bm{\theta}^*} \leq \lambda_{\max}(\bm{M}_t^{-1}) \leq \frac{1}{\lambda_{\min}(\bm{M}_t)} \leq \frac{1}{\lambda}.
\end{align*}

Therefore, we have
\begin{align*}
    \lambda \ab\|\bm{x}_a\|_{\bm{M}_t^{-1}} \ab\|\bm{\theta}^*\|_{\bm{M}_t^{-1}} &\leq \lambda \ab\|\bm{x}_a\|_{\bm{M}_t^{-1}} \ab\|\bm{\theta}^*\|_2\\
    &\leq \lambda \ab\|\bm{x}_a\|_{\bm{M}_t^{-1}} \sqrt{\frac{1}{\lambda}} = \sqrt{\lambda} \ab\|\bm{x}_a\|_{\bm{M}_t^{-1}}. \numberthis\label{eq:bound-first-term}
\end{align*}

For the second term, from Theorem 1 in~\citet{abbasi2011improved}, for any \(\delta \in (0,1)\), with probability at least \(1-\delta\), for all \(t\geq 1\), we have
\begin{align*}
    \ab\|\sum_{s=1}^{t-1} \bm{x}_{a_s} \eta_s + \sum_{s=1}^{t}\sum_{k \in \mathcal{K}_{s}}\doubletilde{\bm{x}}_{k} \tilde{\eta}_s\|_{\bm{M}_t^{-1}} \leq 
    \sqrt{2\log\ab(\frac{\det\ab(\bm{M}_t)^{\frac{1}{2}} \det\ab(\lambda {\bm{I}_d})^{-\frac{1}{2}}}{\delta})}.  \numberthis\label{eq:bound-second-term-1}
\end{align*}

By adopting the determinant-trace inequality (\Cref{lemma:det-trace-inequality}), we have
\begin{align*}
    \Tr\ab(\bm{M}_t) &\leq d\lambda + \sum_{s=1}^{t-1} \Tr(\bm{x}_{a_s}\bm{x}_{a_s}^{\top}) + \sum_{s=1}^{t}\sum_{k \in \mathcal{K}_{s}}\Tr(\doubletilde{\bm{x}}_{k}\doubletilde{\bm{x}}_{k}^{\top})\\
    &\leq d\lambda + t + \ab(1 + \sqrt{d}R) bt,
\end{align*}
which is obtained because there are at most \(bt\) key terms selected by round \(t\) and \(\|\doubletilde{\bm{x}}_k \| \leq  1 + \sqrt{d}R\) for all \(k\in \mathcal{K}\), and therefore,
\begin{align*}
    \det(\bm{M}_t) \leq \ab(\frac{\Tr(\bm{M}_t)}{d})^d \leq \ab(\frac{d\lambda + t + \ab(1 + \sqrt{d}R) bt}{d})^d, \numberthis\label{eq:det-bound}
\end{align*}
where \(\Tr(\bm{X})\) denotes the trace of matrix \(\bm{X}\).

By substituting~\Cref{eq:det-bound} into~\Cref{eq:bound-second-term-1}, we have
\begin{align*}
    &\ab\|\sum_{s=1}^{t-1} \bm{x}_{a_s} \eta_s + \sum_{s=1}^{t}\sum_{k \in \mathcal{K}_{s}}\doubletilde{\bm{x}}_{k} \tilde{\eta}_s\|_{\bm{M}_t^{-1}} \leq \sqrt{2\log\ab(\frac{1}{\delta})+\log\ab(\frac{\det\ab(\bm{M}_t)}{\det(\lambda {\bm{I}_d})})}\\
    \leq\ & \sqrt{2\log\ab(\frac{1}{\delta})+d\log\ab(1 + \frac{t + \ab(1 + \sqrt{d}R)bt}{\lambda d})}. \numberthis\label{eq:bound-second-term-2}
\end{align*}

Plugging~\Cref{eq:bound-first-term} and~\Cref{eq:bound-second-term-2} into~\Cref{eq:diff-estimate-true-reward-1}, we have
\begin{align*}
    &\ab|\bm{x}_a^{\top}\ab(\bm{\theta}_t  - \bm{\theta}^*)| \\
    \leq\ & \ab\|\bm{x}_a\|_{\bm{M}_t^{-1}} \ab(\sqrt{\lambda} + \sqrt{2\log\ab(\frac{1}{\delta})+\log\ab(1 + \frac{t + (1 + \sqrt{d}R)bt}{\lambda d})}). \numberthis\label{eq:diff-estimate-true-reward-2}
\end{align*}
which completes the proof.
\end{proof}

\subsection{Proof of~\Cref{lemma:smoothed-eigenvalue-lower-bound}}
\restatelemmasmoothedeigenvaluelowerbound*

\begin{proof}
  Fix a time \(t\), and denote the key term selected at this time as \(k_t\). Although multiple key terms may be selected at each time step, they all satisfy the properties of this lemma. Therefore, we do not distinguish between them and use only a single subscript \(t\). Let \(\vec{Q}\) be a unitary matrix that rotates the estimated preference vector \(\vec{\theta}_{t}\) to align it with the \(x\)-axis, maintaining its length but zeroing out all components except the first component, i.e., \(\vec{Q}\vec{\theta}_{t} = (\|\vec{\theta}_{t}\|, 0, 0, \dots, 0)\).
  Note that such \(\vec{Q}\) always exists because it just rotates the space.
  According to \conlinucbsk's key term selection strategy \(\doubletilde{\vec{x}}_{k_t} = \argmax_{k \in \mathcal{K}}\vec{\theta}_{t}^{\top} \doubletilde{\vec{x}}_{k}\), we have
  \begin{align*}
    &\lambda_{\min}\ab(\E\ab[\doubletilde{\vec{x}}_{k_t}\doubletilde{\vec{x}}_{k_t}^{\top}])
    = \lambda_{\min}\ab(\E\ab[\vec{x}\vec{x}^{\top} \Bigm| \vec{x} = \argmax_{k \in \mathcal{K}} \vec{\theta}_{t}^{\top} \doubletilde{\vec{x}}_{k}])\\
    =& \min_{\vec{w}:\|\vec{w}\|=1} \vec{w}^{\top} \E\ab[\vec{x}\vec{x}^{\top} \Bigm| \vec{x} = \argmax_{k \in \mathcal{K}} \vec{\theta}_{t}^{\top} \doubletilde{\vec{x}}_{k}] \vec{w}\\
    =& \min_{\vec{w}:\|\vec{w}\|=1} \E\ab[(\vec{w}^{\top}\vec{x})^2 \Bigm| \vec{x} = \argmax_{k \in \mathcal{K}} \vec{\theta}_{t}^{\top} \doubletilde{\vec{x}}_{k}]\\
    \geq& \min_{\vec{w}:\|\vec{w}\|=1} \Var\ab[\vec{w}^{\top}\vec{x} \Bigm| \vec{x} = \argmax_{k \in \mathcal{K}} \vec{\theta}_{t}^{\top} \doubletilde{\vec{x}}_{k}]\\
    =& \min_{\vec{w}:\|\vec{w}\|=1} \Var\ab[(\vec{Q}\vec{w})^{\top}\vec{Q}\vec{x} \Bigm| \vec{x} = \argmax_{k \in \mathcal{K}} (\vec{Q}\vec{\theta}_{t})^{\top} \vec{Q}\doubletilde{\vec{x}}_{k}]\numberthis\label{eq:unitary-property}\\
    =& \min_{\vec{w}:\|\vec{w}\|=1} \Var\ab[\vec{w}^{\top}\vec{Q}\vec{x} \Bigm| \vec{x} = \argmax_{k \in \mathcal{K}} \ab\|\vec{\theta}_{t}\| (\vec{Q}\doubletilde{\vec{x}}_{k})_{1}]\numberthis\label{eq:apply-Q}\\
    =& \min_{\vec{w}:\|\vec{w}\|=1} \Var\ab[\vec{w}^{\top}\vec{Q}\vec{\varepsilon} \Bigm| \vec{\varepsilon} = \argmax_{\vec{\varepsilon}_{k}: k \in \mathcal{K}} (\vec{Q}\tilde{\vec{x}}_{k} + \vec{Q}\vec{\varepsilon}_{k})_{1}]\numberthis\label{eq:decompose-arm-to-mu-and-epsilon}\\
    =& \min_{\vec{w}:\|\vec{w}\|=1} \Var\ab[\vec{w}^{\top}\vec{\varepsilon} \Bigm| \vec{\varepsilon} = \argmax_{\vec{\varepsilon}_{k}: k \in \mathcal{K}} (\vec{Q}\tilde{\vec{x}}_{k} + \vec{\varepsilon}_{k})_{1}]\numberthis\label{eq:remove-Q-by-rotation-invariance}\\
  \end{align*}
  where \Cref{eq:unitary-property} uses the property of unitary matrices: \(\vec{Q}^{\top}\vec{Q}=\vec{I}_d\).
  \Cref{eq:apply-Q} applies matrix \(\vec{Q}\) so only the first component is non-zero and we use the fact that minimizing over \(\vec{Q}\vec{w}\) is equivalent to over \(\vec{w}\).
  \Cref{eq:decompose-arm-to-mu-and-epsilon} follows because each smoothed key term \(\doubletilde{\vec{x}}_k = \tilde{\vec{x}}_k + \vec{\varepsilon}_k\) by definition, and adding a constant a to a random variable does not change its variance.
  \Cref{eq:remove-Q-by-rotation-invariance} is due to the rotation invariance of symmetrically truncated Gaussian distributions.

  Since \(\vec{\varepsilon}_{k} \sim \mathcal{N}(0, \rho^2 \cdot \vec{I}_d)\) conditioned on \(|(\vec{\varepsilon}_k)_j| \leq R, \forall j \in [d]\), by the property of (truncated) multivariate Gaussian distributions, the components of \(\vec{\varepsilon}_{t,i}\) can be equivalently regarded as \(d\) independent samples from a (truncated) univariate Gaussian distribution, i.e., \((\vec{\varepsilon}_k)_j \sim \mathcal{N}(0,\rho^2)\) conditioned on \(|(\vec{\varepsilon}_{k})_j| \leq R, \forall j \in [d]\).
  Therefore, we have
  \begin{align*}
    \Var\left[\vec{w}^{\top}\vec{\varepsilon}\right]
    = \Var\left[\sum_{i=1}^{d} \vec{w}_i \vec{\varepsilon}_i\right]
    = \sum_{i=1}^{d} \vec{w}_i^2\Var\left[ \vec{\varepsilon}_i\right],
  \end{align*}
  where the exchanging of variance and summation is due to the independence of \(\vec{\varepsilon}_i\). Therefore, we can write
  \begin{align*}
    & \min_{\vec{w}:\|\vec{w}\|=1} \Var\ab[\vec{w}^{\top}\vec{\varepsilon} \Bigm| \vec{\varepsilon} = \argmax_{\vec{\varepsilon}_{k}: k \in \mathcal{K}}\ab((\vec{\varepsilon}_{k})_{1} + (\vec{Q}\tilde{\vec{x}}_{k})_{1})]\\
    =&\min_{\vec{w}:\|\vec{w}\|=1} \sum_{j=1}^{d} \vec{w}_j^2\Var\ab[(\vec{\varepsilon})_j \Bigm| \vec{\varepsilon} = \argmax_{\vec{\varepsilon}_{k}: k \in \mathcal{K}}\ab((\vec{\varepsilon}_{k})_{1} + (\vec{Q}\tilde{\vec{x}}_{k})_{1})]\\
    =&\min_{\vec{w}:\|\vec{w}\|=1} \left\{\vec{w}_1^2\Var\ab[(\vec{\varepsilon})_1 \Bigm| \vec{\varepsilon} = \argmax_{\vec{\varepsilon}_{k}: k \in \mathcal{K}}\ab((\vec{\varepsilon}_{k})_{1} + (\vec{Q}\tilde{\vec{x}}_{k})_{1})]\right.\\
    &\left.+ \sum_{j=2}^{d} \vec{w}_j^2 \Var\ab[(\vec{\varepsilon})_j \Bigm| \vec{\varepsilon} = \argmax_{\vec{\varepsilon}_{k}: k \in \mathcal{K}}\ab((\vec{\varepsilon}_{k})_{1} + (\vec{Q}\tilde{\vec{x}}_{k})_{1})]\right\}\\
    =&\min_{\vec{w}:\|\vec{w}\|=1} \left\{\vec{w}_1^2\Var\ab[(\vec{\varepsilon})_1 \Bigm| \vec{\varepsilon} = \argmax_{\vec{\varepsilon}_{k}: k \in \mathcal{K}}\ab((\vec{\varepsilon}_{k})_{1} + (\vec{Q}\tilde{\vec{x}}_{k})_{1})]\right.\\
    &\left.+ \sum_{j=2}^{d} \vec{w}_j^2 \Var\ab[(\vec{\varepsilon})_j]\right\}\\
    =&\min_{\vec{w}:\|\vec{w}\|=1} \left\{\vec{w}_1^2\Var\ab[(\vec{\varepsilon})_1 \Bigm| \vec{\varepsilon} = \argmax_{\vec{\varepsilon}_{k}: k \in \mathcal{K}}\ab((\vec{\varepsilon}_{k})_{1} + (\vec{Q}\tilde{\vec{x}}_{k})_{1})] + (1-\vec{w}_1^2) \rho^2\right\}\\
    =&\min \left\{\Var\ab[(\vec{\varepsilon})_1 \Bigm| \vec{\varepsilon} = \argmax_{\vec{\varepsilon}_{k}: k \in \mathcal{K}}\ab((\vec{\varepsilon}_{k})_{1} + (\vec{Q}\tilde{\vec{x}}_{k})_{1})], \rho^2\right\} \geq c_1 \frac{\rho^2}{\log |\mathcal{K}|},
  \end{align*}
  where in the last inequality, we use Lemma 15 and Lemma 14 in \citet{sivakumar2020structured} and get
  \begin{align*}
    &\Var\ab[(\vec{\varepsilon})_1 \Bigm| \vec{\varepsilon} = \argmax_{\vec{\varepsilon}_{k}: k \in \mathcal{K}}\ab((\vec{\varepsilon}_{k})_{1} + (\vec{Q}\tilde{\vec{x}}_{k})_{1})]\\
    \geq& \Var\ab[(\vec{\varepsilon})_1 \Bigm| \vec{\varepsilon} = \argmax_{\vec{\varepsilon}_{k}: k \in \mathcal{K}} (\vec{\varepsilon}_{k})_{1}] \geq c_1\frac{\rho^2}{\log |\mathcal{K}|}.\tag*{\qedhere}
  \end{align*}
\end{proof}

\subsection{Proof of~\Cref{lemma:smoothed-linear-t}}
\restatelemmasmoothedlineart*
\begin{proof}
  To apply the matrix Chernoff bound (\Cref{lemma:matrix-chernoff}), we first verify the required two conditions for the self-adjoint matrices \(\doubletilde{\vec{x}}_{k}\doubletilde{\vec{x}}_{k}^{\top}\) for any \(k \in \mathcal{K}_s\) and \(s \in [t]\).
  First, \(\doubletilde{\vec{x}}_{k}\doubletilde{\vec{x}}_{k}^{\top}\) is obviously positive semi-definite.
  Second, by the Courant-Fischer theorem,
  \begin{align*}
    \lambda_{\max}(\doubletilde{\vec{x}}_{k}\doubletilde{\vec{x}}_{k}^{\top}) &= \max_{\vec{w}: \|\vec{w}\|=1} \vec{w}^{\top}\doubletilde{\vec{x}}_{k}\doubletilde{\vec{x}}_{k}^{\top} \vec{w} = \max_{\vec{w}: \|\vec{w}\|=1} (\vec{w}^{\top}\doubletilde{\vec{x}}_{k})^2\\
    &\leq \max_{\vec{w}: \|\vec{w}\|=1}\|\vec{w}\|^2\|\doubletilde{\vec{x}}_{k}\|^2 \leq (1+\sqrt{d}R)^2.
  \end{align*}
  Next, by \Cref{lemma:smoothed-eigenvalue-lower-bound} and the super-additivity of the minimum eigenvalue (due to Weyl's inequality), we have
  \begin{align*}
    \mu_{\min} = \lambda_{\min}\ab(\sum_{s=1}^{t}\sum_{k \in \mathcal{K}_{s}} \E\ab[\doubletilde{\bm{x}}_k \doubletilde{\bm{x}}_k^{\top}]) \geq \sum_{s=1}^{t}\sum_{k \in \mathcal{K}_{s}} \lambda_{\min}\ab(\E\ab[\doubletilde{\bm{x}}_k \doubletilde{\bm{x}}_k^{\top}]) \geq \lambda_{\mathcal{K}}b t ,
  \end{align*}
  where the last inequality is because there are at most \(bt\) key terms selected by round \(t\), so the summation has at most \(bt\) terms.
  So by \Cref{lemma:matrix-chernoff}, we have for any \(\varepsilon \in (0,1)\),
  \begin{align*}
    &\Pr\left[\lambda_{\min}\ab(\sum_{s=1}^{t}\sum_{k \in \mathcal{K}_{s}} \doubletilde{\bm{x}}_k \doubletilde{\bm{x}}_k^{\top}) \leq (1-\varepsilon) \lambda_{\mathcal{K}} bt \right]\\
    \leq& \Pr\left[\lambda_{\min}\ab(\sum_{s=1}^{t}\sum_{k \in \mathcal{K}_{s}} \doubletilde{\bm{x}}_k \doubletilde{\bm{x}}_k^{\top}) \leq (1-\varepsilon)\mu_{\min}\right]\\
    \leq& d \ab[\frac{e^{-\varepsilon}}{(1-\varepsilon)^{1-\varepsilon}}]^{\mu_{\min}/(1+\sqrt{d}R)^2}\\
    \leq& d \ab[\frac{e^{-\varepsilon}}{(1-\varepsilon)^{1-\varepsilon}}]^{\frac{\lambda_{\mathcal{K}}bt}{(1+\sqrt{d}R)^2}},
  \end{align*}
  where the last inequality is because \(e^{-x}\) is decreasing.
  Choosing \(\varepsilon=\frac{1}{2}\), we get
  \begin{align*}
    \Pr\left[\lambda_{\min}\ab(\sum_{s=1}^{t}\sum_{k \in \mathcal{K}_{s}} \doubletilde{\bm{x}}_k \doubletilde{\bm{x}}_k^{\top}) \leq \frac{\lambda_{\mathcal{K}} bt}{2}\right] \leq d \ab(\sqrt{2}e^{-\frac{1}{2}})^{\frac{\lambda_{\mathcal{K}}bt}{(1+\sqrt{d}R)^2}}.
  \end{align*}
  Letting the RHS be \(\delta\), we get \(t = \frac{2(1+\sqrt{d}R)^2\log(\frac{d}{\delta})}{\lambda_{\mathcal{K}}b(1-\log(2))} \leq \frac{8(1+\sqrt{d}R)^2}{\lambda_{\mathcal{K}}b}\log\ab(\frac{d}{\delta})\).
  Therefore, \(\lambda_{\min}\ab(\sum_{s=1}^{t}\sum_{k \in \mathcal{K}_{s}} \doubletilde{\bm{x}}_k \doubletilde{\bm{x}}_k^{\top}) \geq \frac{\lambda_{\mathcal{K}}bt}{2}\) holds with probability at least \(1-\delta\) when \(t \geq \frac{8(1+\sqrt{d}R)^2}{\lambda_{\mathcal{K}}b}\log\ab(\frac{d}{\delta})\).
\end{proof}

\subsection{Proof of~\Cref{lemma:smoothed-bounded-x-norm}}
\restatelemmasmoothedboundedxnorm*
\begin{proof}
\begin{align*}
    \ab\|\bm{x}_a\|_{\bm{M}_t^{-1}} = \sqrt{\bm{x}_a^{\top}\bm{M}_t^{-1}\bm{x}_a} \leq \sqrt{\lambda_{\max}(\bm{M}_t^{-1}) \bm{x}_a^{\top}\bm{x}_a} = \sqrt{\frac{1}{\lambda_{\min}(\bm{M}_t)}}, \numberthis\label{eq:bound-x-norm}
\end{align*}
where the first inequality is due to the property of the Rayleigh quotient, and the second inequality is due to the fact that \(\bm{x}_a^{\top}\bm{x}_a = 1\).

By the definition of \(\bm{M}_t\), we have
\begin{align*}
    \lambda_{\min}(\bm{M}_t) & = \lambda_{\min}\ab(\sum_{s=1}^{t-1} \bm{x}_{a_s}\bm{x}_{a_s}^{\top}  + \sum_{s=1}^{t}\sum_{k \in \mathcal{K}_{s}}\doubletilde{\bm{x}}_{k}\doubletilde{\bm{x}}_{k}^{\top} + \lambda \bm{I}_d\ab) \\
    &\geq  \lambda_{\min}\ab(\sum_{s=1}^{t}\sum_{k \in \mathcal{K}_{s}}\doubletilde{\bm{x}}_{k}\doubletilde{\bm{x}}_{k}^{\top}) \\
    & \geq \frac{\lambda_{\mathcal{K}}bt}{2}, \numberthis\label{eq:lower-bound-Mt}
\end{align*}
where the first inequality follows the property of Loewner order that if \(\bm{A} \succeq \bm{B}\) then \(\lambda_{\min}(\bm{A}) \geq \lambda_{\min}(\bm{B})\), and the last inequality follows from \Cref{lemma:smoothed-linear-t} conditioned on \(t \geq T_0\).

Therefore, by plugging~\Cref{eq:lower-bound-Mt} into~\Cref{eq:bound-x-norm}, we have
\(\ab\|\bm{x}_a\|_{\bm{M}_t^{-1}} \leq \sqrt{\frac{2}{\lambda_{\mathcal{K}}bt}}\).
\end{proof}

\subsection{Proof of~\Cref{lemma:me-diff-estimate-true-reward}}
\label{app:lemma:me-diff-estimate-true-reward}
\restatelemmamediffestimatetruereward*

\begin{proof}
The proof of \Cref{lemma:me-diff-estimate-true-reward} is similar to that of \Cref{lemma:smoothed-diff-estimate-true-reward}. The only difference is the trace and determinant of the matrix \(\bm{M}_t\).

We first show that by round \(t\), at most \(\frac{\alpha dt}{c_0^2}\) key terms have been selected since the beginning of the algorithm for all three uncertainty checking functions.

Consider the case where \conlinucbme uses the ``Continuous Checking''  function, i.e., the agent checks the eigenvalues of the matrix \(\bm{M}_t\) at each round.
We first denote the covariance matrix before selecting key terms at round \(t\) by \(\bm{M}_{t}^{\prime}\), i.e., \(\bm{M}_{t} = \bm{M}_{t}^{\prime} + \sum_{k\in \mathcal{K}_t} \bm{x}_k\bm{x}_k^{\top}\).
For \(\bm{M}_{t}^{\prime}\), denote its eigenvectors by \(\{\bm{v}_i\}_{i=1}^d\) and corresponding eigenvalues by \(\{\lambda_{\bm{v}_i}\}_{i=1}^d\).
If some key term \(k\) is selected at round \(t\), then there must exist an eigenvector \(\bm{v}_i\) such that \(\lambda_{\bm{v}_i} < \alpha t\), and the corresponding key term context \(\tilde{\bm{x}}_k\) is close to \(\bm{v}_i\), i.e., \(\tilde{\bm{x}}_k^{\top} \bm{v}_i \ge c_0\).

We can write the vector \(\tilde{\bm{x}}_k = \sum_{i=1}^d \gamma_i \bm{v}_i\) for some coefficients \(\{\gamma_i\}_{i=1}^d\). 
Then, for \(j\in[d]\), Denote \( \bm{z}_j = \sum_{i=1, i\neq j}^d \gamma_i \bm{v}_i\).
For any \(j\in[d]\), we have \(\tilde{\bm{x}}_k^{\top} \bm{v}_j = \sum_{i=1}^d \gamma_i \bm{v}_i^{\top} \bm{v}_j = \gamma_j \ge c_0\) and 
\(\tilde{\bm{x}}_k  \tilde{\bm{x}}_k ^{\top} = (\gamma_j\bm{v}_j + \bm{z}_j)(\gamma_j\bm{v}_j + \bm{z}_j)^{\top} = \gamma_j^2 \bm{v}_j\bm{v}_j^{\top} + \bm{z}_j\bm{z}_j^{\top}\), and then, we have the following:
\begin{align*}
    &\bm{M}_{t}^{\prime} + \sum_{k\in \mathcal{K}_t} \bm{x}_k\bm{x}_k^{\top} =\bm{M}_{t}^{\prime} + \sum_{i\in[d]: \lambda_{\bm{v}_i} \leq \alpha t} \left\lceil\frac{\alpha t - \lambda_{\bm{v}_i}}{c_0^2}\right\rceil( \gamma_i^2 \bm{v}_i\bm{v}_i^{\top} + \bm{z}_i\bm{z}_i^{\top}) \\
    \succeq\ & \sum_{i=1}^d \lambda_{\bm{v}_i} \bm{v}_i\bm{v}_i^{\top} + \sum_{i\in[d]: \lambda_{\bm{v}_i} \leq \alpha t} \frac{\alpha t - \lambda_{\bm{v}_i}}{c_0^2}( \gamma_i^2 \bm{v}_i\bm{v}_i^{\top} + \bm{z}_i\bm{z}_i^{\top}) \\
    \succeq\ & \sum_{i=1}^d \lambda_{\bm{v}_i} \bm{v}_i\bm{v}_i^{\top} + \sum_{i\in[d]: \lambda_{\bm{v}_i} \leq \alpha t}  \ab(\alpha t - \lambda_{\bm{v}_i})\bm{v}_i\bm{v}_i^{\top} \\
    =\ & \sum_{i\in[d]:\lambda_{\bm{v}_i} < \alpha t} (\alpha t - \lambda_{\bm{v}_i} +  \lambda_{\bm{v}_i}) \bm{v}_i\bm{v}_i^{\top} + \sum_{i\in[d]: \lambda_{\bm{v}_i} > \alpha t} \lambda_{\bm{v}_i} \bm{v}_i\bm{v}_i^{\top} \\
    \succeq\ & \sum_{i=1}^d \alpha t \bm{v}_i\bm{v}_i^{\top}. \numberthis\label{eq:lower-bound-Mt-prime-continuous}
\end{align*}

Following from \Cref{eq:lower-bound-Mt-prime-continuous}, we have
\begin{align*}
    \lambda_{\min}(\bm{M}_t) \geq \alpha t. \numberthis\label{eq:lower-bound-Mt-prime-continuous-2}
\end{align*}

Denote the number of key terms selected at round \(t\) as \(K_t\). We have \(K_t = \sum_{i=1}^{d} \frac{\alpha t - \lambda_{\bm{v}_i}}{c_0^2}\). 
Since \(\lambda_{\bm{v}_i} \geq \alpha(t-1), \forall i \in [d]\) according to \Cref{eq:lower-bound-Mt-prime-continuous-2}, we have \(K_t \leq \frac{\alpha d}{c_0^2}\), and then \(\sum_{s=1}^{t} K_s \leq \frac{\alpha dt}{c_0^2}\).

For the ``Fixed Interval Checking'' function, at each uncertainty checking point \(t_j = jP\) where \(j\in\{1,2,\dots, \lfloor \frac{T}{P} \rfloor\} \), we have \(\lambda_{\min}(\bm{M}_{t_j}) \geq \alpha t_j \). 
For the \(j\)-th checking, there are \(\sum_{i=1}^{d} \frac{\alpha t_j - \lambda_{\bm{v}_i}}{c_0^2} \leq\sum_{i=1}^{d} \frac{\alpha t_j - \alpha t_{j-1}}{c_0^2} \leq \frac{\alpha d P}{c_0^2}\) conversations to be launched.
Thus, by round \(t\), the number of total conversations satisfies \( \sum_{j=1}^{\lfloor \frac{t}{P} \rfloor} \frac{\alpha d P}{c_0^2} \leq \frac{\alpha dt}{c_0^2}\).

For the ``Exponential Phase Checking'' function, by round \(t\), there are \(\lfloor\log_2(t)\rfloor\) uncertainty checking points. 
For the \(j\)-th checking, there are \(\sum_{i=1}^{d} \frac{\alpha t_j - \lambda_{\bm{v}_i}}{c_0^2} \leq\sum_{i=1}^{d} \frac{\alpha t_j - \alpha t_{j-1}}{c_0^2} \leq \frac{\alpha d 2^{j-1}}{c_0^2}\) conversations to be launched.
By round \(t\), the number of total conversations satisfies \( \sum_{j=1}^{\lfloor \log_2(t) \rfloor} \frac{\alpha d 2^{j-1}}{c_0^2} \leq \frac{\alpha dt}{c_0^2}\).

Therefore, we have
\begin{align*}
    \Tr(\bm{M}_t) \leq d\lambda + \sum_{s=1}^{t-1} \Tr(\bm{x}_{a_s}\bm{x}_{a_s}^{\top}) + \sum_{s=1}^{t}\sum_{k \in \mathcal{K}_{s}}\Tr(\tilde{\bm{x}}_{k}\tilde{\bm{x}}_{k}^{\top}) \leq d\lambda + t + \frac{\alpha dt}{c_0^2},
\end{align*}
and
\begin{align*}
    \det(\bm{M}_t) \leq \ab(\frac{\Tr(\bm{M}_t)}{d})^d \leq \ab(\frac{d\lambda + t + \frac{\alpha d t}{c_0^2}}{d})^d \leq \ab(\lambda +\frac{t+\alpha dt}{ c_0^2 d})^d,
\end{align*}
where the last inequality is obtained by the fact that \(c_0 < 1 \).

Following the same steps as in the proof of \Cref{lemma:smoothed-diff-estimate-true-reward}, we can obtain that
\begin{align*}
    \ab|\bm{x}_a^{\top}\ab(\bm{\theta}_t  - \bm{\theta}^*)| \leq  \ab\|\bm{x}_a\|_{\bm{M}_t^{-1}} \ab(\sqrt{\lambda} + \sqrt{2\log\ab(\frac{1}{\delta})+d\log\ab(1 + \frac{t + \alpha dt}{\lambda d c_0^2 })}),
\end{align*}
which concludes the proof.
\end{proof}

\subsection{Proof of~\Cref{lemma:me-bounded-x-norm}}
\label{app:lemma:me-bounded-x-norm}
\restatelemmameboundedxnorm*

\begin{proof}
We first consider the case where \conlinucbme uses the ``Continuous Checking''  function, i.e., the agent checks the eigenvalues of the matrix \(\bm{M}_t\) at each round.
By \Cref{eq:lower-bound-Mt-prime-continuous-2}, we have  \(\lambda_{\min}(\bm{M}_t) \geq \alpha t\).
Then, following from \Cref{eq:bound-x-norm} in the proof of \Cref{lemma:smoothed-bounded-x-norm}, we can obtain that \(\ab\|\bm{x}_a\|_{\bm{M}_t^{-1}} \leq \sqrt{\frac{1}{\alpha t}}\).

Next, we consider the case for the ``Fixed Interval Checking'' function.
In this case, the agent only checks the eigenvalues of the matrix \(\bm{M}_t\) every \(P\) rounds.
For the rounds \(t\) when the agent checks the uncertainty, we have the same results as \(\lambda_{\min}(\bm{M}_t) \geq \alpha t\);
For the rounds \(t\) when the agent does not check it, we have \(\lambda_{\min}(\bm{M}_t) \geq \alpha t^{\prime}\) where \(t^{\prime}\) is the last round that the agent conducts the check and \(t - t^{\prime} \leq P\).
When \(t\geq 2P\), \( t^{\prime} \geq t-P \ge \frac{t}{2}\), we can obtain that \(\ab\|\bm{x}_a\|_{\bm{M}_t^{-1}} \leq \sqrt{\frac{1}{\alpha t^{\prime}}} \leq \sqrt{\frac{2}{\alpha t}}\).

Finally, we consider the ``Exponential Phase Checking'' function.
At rounds \(t\) satisfying \(2^i \leq t < 2^{i+1}\)  for \(i=1,2,\dots\), the last checking point \(t^{\prime} = 2^i\), then we have \(\lambda_{\min}(\bm{M}_t) \geq \alpha \cdot 2^i\).
When \(t \geq 2\), we have \(\frac{t}{2} \leq 2^i\), and then  \(\ab\|\bm{x}_a\|_{\bm{M}_t^{-1}} \leq \sqrt{\frac{1}{\alpha 2^i}} \leq \sqrt{\frac{2}{\alpha t}}\).

Therefore, to generalize the bound,  we can conclude that when \(t\geq 2P\), \(\ab\|\bm{x}_a\|_{\bm{M}_t^{-1}} \leq \sqrt{\frac{2}{\alpha t}}\) for all three checking functions.
\end{proof}

\subsection{Proof of~\Cref{theorem:conlinucbsk-regret}}
\label{app:proof-conlinucbsk-regret}
\restateregretconlinucbsk*
\begin{proof}
Denote the instantaneous regret at round \(t\) by \(\text{reg}_t\).
We first decompose it as follows:
\begin{align*}
    &\text{reg}_t  = (\bm{x}_{a_t^*}^{\top}\bm{\theta}^* + \eta_t) - (\bm{x}_{a_t}^{\top}\bm{\theta}^* + \eta_t)  \\
    =& \bm{x}_{a_t^*}^{\top}(\bm{\theta}^* - \bm{\theta}_t) + (\bm{x}_{a_t^*}^{\top}\bm{\theta}_t + \alpha_t \| \bm{x}_{a_t^*} \|_{\bm{M}_t^{-1}}) - (\bm{x}_{a_t}^{\top}\bm{\theta}_t + \alpha_t \| \bm{x}_{a_t} \|_{\bm{M}_t^{-1}})\\
    &+ \bm{x}_{a_t}^{\top}( \bm{\theta}_t - \bm{\theta}^*) - \alpha_t \| \bm{x}_{a_t^*} \|_{\bm{M}_t^{-1}} +\alpha_t \| \bm{x}_{a_t} \|_{\bm{M}_t^{-1}} \\
    \leq& \bm{x}_{a_t^*}^{\top}(\bm{\theta}^* - \bm{\theta}_t)  + \bm{x}_{a_t}^{\top}( \bm{\theta}_t - \bm{\theta}^*) - \alpha_t \| \bm{x}_{a_t^*} \|_{\bm{M}_t^{-1}} + \alpha_t \| \bm{x}_{a_t} \|_{\bm{M}_t^{-1}}  \numberthis\label{eq:use-ucb} \\
    \leq& \alpha_t \| \bm{x}_{a_t^*} \|_{\bm{M}_t^{-1}} + \alpha_t \| \bm{x}_{a_t} \|_{\bm{M}_t^{-1}}  - \alpha_t \| \bm{x}_{a_t^*} \|_{\bm{M}_t^{-1}} + \alpha_t \| \bm{x}_{a_t} \|_{\bm{M}_t^{-1}} \numberthis \label{eq:use-lemma-smoothed-diff-estimate-true-reward} \\
    \leq& 2\alpha_t \| \bm{x}_{a_t} \|_{\bm{M}_t^{-1}},
\end{align*}
where  \Cref{eq:use-ucb} follows from the UCB strategy for arm selection, and \Cref{eq:use-lemma-smoothed-diff-estimate-true-reward} follows from \Cref{lemma:smoothed-diff-estimate-true-reward}.
Next, we have
\begin{align*}
    \R(T) & = \sum_{t=1}^{T_0} \text{reg}_t + \sum_{t=T_0+1}^{T} \text{reg}_t \\
    &\leq T_0 + \sum_{t=T_0+1}^{T} 2\alpha_t \| \bm{x}_{a_t} \|_{\bm{M}_t^{-1}} \numberthis \label{eq:regret_less-than-1}\\
    & \leq T_0 + 2\sum_{t=T_0+1}^{T} \alpha_t \sqrt{\frac{2}{\lambda_{\mathcal{K}}bt}} \numberthis\label{eq:use-smooth-bound-x} \\
    & \leq T_0 + 4\alpha_T \sqrt{\frac{2T}{\lambda_{\mathcal{K}}b}} \numberthis\label{eq:simple-calculation} \\
\end{align*}
where \Cref{eq:regret_less-than-1} is because the instantaneous regret \(\text{reg}_t\leq 1\) by \Cref{assumption:normalized-vector},
\Cref{eq:use-smooth-bound-x} follows from \Cref{lemma:smoothed-bounded-x-norm}, and \Cref{eq:simple-calculation} is because \(\alpha_t\) is non-decreasing and \(\sum_{t=1}^{T} \frac{1}{\sqrt{t}} \leq 2\sqrt{T}\).

Recall the definition of \(T_0\triangleq \frac{8(1+\sqrt{d}R)^2}{b\lambda_{\mathcal{K}}}\log\ab(\frac{d}{\delta})\) in \Cref{lemma:smoothed-linear-t} and the definition of \(\alpha_t\) in \Cref{lemma:smoothed-diff-estimate-true-reward}.
Plugging \(T_0\) and \(\alpha_t\) into \Cref{eq:simple-calculation}, we can obtain the regret bound.
\end{proof}

\subsection{Proof of~\Cref{theorem:conlinucbme-regret}}
\label{app:proof-conlinucbme-regret}
\restateregretconlinucbme*
\begin{proof}
    With the same decomposition as in the proof of \Cref{theorem:conlinucbsk-regret}, we have 
    \begin{align*}
        &\R(T) = \sum_{t=1}^{2P} \text{reg}_t +  \sum_{t=2P+1}^{T} \text{reg}_t \leq 2P + 2\sum_{t=2P+1}^{T} \alpha_t \| \bm{x}_{a_t} \|_{\bm{M}_t^{-1}} \\ 
        \leq\ &2P + 2 \sum_{t=2P+1}^{T} \alpha_t \sqrt{\frac{2}{\alpha t}} \numberthis\label{eq:me-use-me-diff} \\
        \leq\ &2P + 4\alpha_T\sqrt{\frac{2T}{\alpha}}. \numberthis\label{eq:me-regret-bound} \\
        =\ &2P + 4\sqrt{\frac{2T}{\alpha}} \ab(\sqrt{2\log{\ab(\frac{1}{\delta})}+d\log\ab(1+\frac{T + \alpha dT}{\lambda d c_0^2})} + \sqrt{\lambda}),\numberthis\label{eq:plug-in-alpha}
    \end{align*}
    where \Cref{eq:me-use-me-diff,eq:me-regret-bound} follow from \Cref{lemma:me-diff-estimate-true-reward} and analogous steps in \Cref{theorem:conlinucbsk-regret}. 
    Note that \(P>1\) is a given constant for the ``Fixed Interval Checking'' function.
    Plugging \(\alpha_T\) into the inequality, we can obtain the result and conclude that \(\R(T) = \mathcal{O}(\sqrt{dT\log(T)}). \)
\end{proof}

\subsection{Proof of \Cref{thm:lowerbound}}
\label{sec:lower-bound}
  Since any algorithms for conversational bandits must select both arms and key terms, we model a policy \(\pi\) as a tuple consisting of two components \(\pi = (\pi^{\text{arm}}, \pi^{\text{key}})\), where \(\pi^{\text{arm}}\) selects arms and \(\pi^{\text{key}}\) selects key terms.
  We assume that at each time step, the policy can select at most one key term; otherwise, the number of key terms could exceed the number of arms, which is impractical.
  Let \(\mathcal{H}_t = \set{a_1, x_1, k_1, \widetilde{x}_1, \dots,a_{t}, x_{t}, k_{t}, \widetilde{x}_{t}}\) denote the history of interactions between the policy and the environment up to time \(t\).
  We note that the presence of key terms at every time step in \(\mathcal{H}_t\) is without loss of generality because we allow \(k_t\) to be empty if no conversation is initiated at round \(t\).
  The noise terms associated with both arm-level and key term-level feedback, denoted by \(\eta_{t}\) and \(\widetilde{\eta}_{t}\), follow the standard Gaussian distribution \(\mathcal{N}(0,1)\).
  We also denote the feature vectors of selected arm and key term as random variables \(\vec{A}_t, \vec{K}_t \in \RR^d\), and the arm-level and key term-level rewards \(X_t = \inprod{\vec{A}_t}{\vec{\theta}} + \eta_{t}\) and \(\widetilde{X}_t = \inprod{\vec{K}_t}{\vec{\theta}} + \widetilde{\eta}_{t}\), follow \(\mathcal{N}(\inprod{\vec{A}_t}{\vec{\theta}},1)\) and \(\mathcal{N}(\inprod{\vec{K}_t}{\vec{\theta}},1)\), respectively.
  We denote by \(\PP_{\vec{\theta}}\) the probability measure induced by the environment \(\vec{\theta}\) and policy \(\pi\), and by \(\E_{\vec{\theta}}\) the expectation under  \(\PP_{\vec{\theta}}\).
  With these definitions, we present the following lemma.

\begin{lemma}\label{lemma:divergence-decomposition}
  Let \(D(P \parallel Q)\) denote the KL divergence between distributions \(P\) and \(Q\), and let \(\vec{\theta}\), \(\vec{\theta}'\) be two environments, then we have
  \[D(\PP_{\vec{\theta}} \parallel \PP_{\vec{\theta}'}) = \frac{1}{2}\sum_{t=1}^T \ab(\E\nolimits_{\vec{\theta}}\ab[\inprod{\vec{A}_t}{\vec{\theta}-\vec{\theta}'}^2] + \E\nolimits_{\vec{\theta}}\ab[\inprod{\vec{K}_t}{\vec{\theta}-\vec{\theta}'}^2]).\]
\end{lemma}

  \begin{proof}
    Given a bandit instance with parameter \(\vec{\theta}\) and a policy \(\pi\), according to Section 4.6 of \citet{lattimore-2020-bandit-algorithms}, we construct the canonical bandit model of our setting as follows.
    Let \((\Omega,\mathcal{F}, \PP_{\vec{\theta}})\) be a probability space and \(\mathcal{A}\) be the set of all possible arms, where \(\Omega=(\mathcal{A} \times \RR)^{T}\), \(\mathcal{F}=\mathcal{B}(\Omega)\), and the density function of the probability measure \(\PP_{\vec{\theta}}\) is defined by \(p_{\vec{\theta},\pi}: \Omega \to \RR\):
    \begin{align*}
      p_{\vec{\theta}}(\mathcal{H}_{T}) = \prod_{t=1}^{T} \pi_t^{\text{arm}}(a_t \mid \mathcal{H}_{t-1}) p_{a_t}(x_t) \cdot \pi_t^{\text{key}}(k_t \mid \mathcal{H}_{t-1}) \widetilde{p}_{k_t}(\widetilde{x}_t),
    \end{align*}
    where \(p_{a_t}\) and \(\widetilde{p}_{k_t}\) are the density functions of arm-level and key term-level reward distributions \(P_{a_t}\) and \(\widetilde{P}_{k_t}\), respectively.
    The definition of \(\PP_{\vec{\theta}'}\) is identical except that \(p_{a_t}\), \(\widetilde{p}_{k_t}\) are replaced by \(p'_{a_t}\), \(\widetilde{p}'_{k_t}\) and \(P_{a_t}\), \(\widetilde{P}_{k_t}\) are replaced by \(P'_{a_t}\), \(\widetilde{P}'_{k_t}\).
    
    By the definition of KL divergence \(D(P \parallel Q)=\int_{\Omega}\log\left(\odv{P}{Q}\right)\odif{P}\),
    \begin{align*}
      D(\PP_{\vec{\theta}} \parallel \PP_{\vec{\theta}'}) =\int_{\Omega}\log\left(\odv{\PP_{\vec{\theta}}}{\PP_{\vec{\theta}'}}\right)\odif{\PP_{\vec{\theta}}}= \E\nolimits_{\vec{\theta}}\left[ \log \odv{\PP_{\vec{\theta}}}{\PP_{\vec{\theta}'}} \right].
    \end{align*}
    
    Note that
    \begin{align*}
      &\log \ab(\odv{\PP_{\vec{\theta}}}{\PP_{\vec{\theta}'}}(\mathcal{H}_T))= \log \frac{p_{\vec{\theta},\pi}(\mathcal{H}_{T})}{p_{\vec{\theta}',\pi}(\mathcal{H}_{T})} \numberthis\label{eq:chain-rule-radon-nikodym}\\
      =& \log \frac{\prod_{t=1}^{T} \pi_t^{\text{arm}}(a_t \mid \mathcal{H}_{t-1}) p_{a_t}(x_t) \cdot \pi_t^{\text{key}}(k_t \mid \mathcal{H}_{t-1}) \widetilde{p}_{k_t}(\widetilde{x}_t)}{\prod_{t=1}^{T} \pi_t^{\text{arm}}(a_t \mid \mathcal{H}_{t-1}) p'_{a_t}(x_t) \cdot \pi_t^{\text{key}}(k_t \mid \mathcal{H}_{t-1}) \widetilde{p}'_{k_t}(\widetilde{x}_t)}\\
      =& \sum_{t=1}^{T} \ab(\log \frac{p_{a_t}(x_t)}{p'_{a_t}(x_t)} + \log \frac{\widetilde{p}_{k_t}(\widetilde{x}_t)}{\widetilde{p}'_{k_t}(\widetilde{x}_t)}).
    \end{align*}
    where in Equation~\ref{eq:chain-rule-radon-nikodym} we used the chain rule for Radon–Nikodym derivatives, and in the last equality, all the terms involving the policy \(\pi\) cancel.
    Therefore,
    \begin{align*}
      &D(\PP_{\vec{\theta}} \parallel \PP_{\vec{\theta}'})\\
      =& \sum_{t=1}^{T} \ab(\E\nolimits_{\vec{\theta}} \ab[\log \frac{p_{A_t}(X_t)}{p'_{A_t}(X_t)}] + \E\nolimits_{\vec{\theta}}\ab[\log \frac{\widetilde{p}_{K_t}(\widetilde{X}_t)}{\widetilde{p}'_{K_t}(\widetilde{X}_t)}])\\
      =& \sum_{t=1}^{T} \ab(\E\nolimits_{\vec{\theta}}\ab[\E\nolimits_{\vec{\theta}} \ab[\log \frac{p_{A_t}(X_t)}{p'_{A_t}(X_t)} \bigg\mid A_t]] + \E\nolimits_{\vec{\theta}}\ab[\E\nolimits_{\vec{\theta}}\ab[\log \frac{\widetilde{p}_{K_t}(\widetilde{X}_t)}{\widetilde{p}'_{K_t}(\widetilde{X}_t)}\bigg\mid K_t]])\\
      =& \sum_{t=1}^{T} \ab(\E\nolimits_{\vec{\theta}}\ab[D(P_{A_t} \parallel P'_{A_t})] + \E\nolimits_{\vec{\theta}}\ab[D(\widetilde{P}_{K_t} \parallel \widetilde{P}'_{K_t})])\\
      =& \frac{1}{2}\sum_{t=1}^T \ab(\E\nolimits_{\vec{\theta}}\ab[\inprod{\vec{A}_t}{\vec{\theta}-\vec{\theta}'}^2] + \E\nolimits_{\vec{\theta}}\ab[\inprod{\vec{K}_t}{\vec{\theta}-\vec{\theta}'}^2]).
    \end{align*}
    where the last equation uses \Cref{lemma:kl-divergence-of-gaussian} and the fact that \(P_{A_t}\sim \mathcal{N}(\inprod{A_t}{\vec{\theta}},1)\), \(P'_{A_t}\sim \mathcal{N}(\inprod{A_t}{\vec{\theta}'},1)\), \(\widetilde{P}_{K_t}\sim \mathcal{N}(\inprod{K_t}{\vec{\theta}},1)\), and \(\widetilde{P}'_{K_t}\sim \mathcal{N}(\inprod{K_t}{\vec{\theta}'},1)\), respectively.
  \end{proof}

Next, we present a lower bound for conversational bandits but \emph{without} imposing the constraint that the number of arms is \(K\).

\begin{lemma}\label{lemma:lower-bound-infinite-arms}
  Let the arm set and the key term set \(\mathcal{A} = \mathcal{K} = [-1,1]^d\) and \(\Theta = \ab\{\pm\sqrt{\frac{1}{T}}\}^d\), then for any policy, there exists an environment \(\vec{\theta} \in \Theta\) such that the expected regret satisfies:
  \(\E\nolimits_{\vec{\theta}}\ab[R(T)] \geq \frac{\exp(-4)}{4} d\sqrt{T}\).
\end{lemma}
\begin{proof}
  For any \(i \in [d]\) and \(\vec{\theta} \in \Theta\), define \(\mathcal{E}_{\vec{\theta},i}\) as the event that the sign of the \(i\)-th coordinate of at least half of \(\{\vec{A}_t\}_{t=1}^T\) does not agree with \(\vec{\theta}\):
  \(\mathcal{E}_{\vec{\theta},i} = \ab\{\sum_{t=1}^T \1\ab\{\sign(\vec{A}_{ti}) \neq \sign(\vec{\theta}_i)\} \geq \frac{T}{2}\}\).
  
  Let \(p_{\vec{\theta},i}=\PP_{\vec{\theta}}\ab[\mathcal{E}_{\vec{\theta},i}]\) and \(\vec{\theta}'=(\vec{\theta}_1, \dots, \vec{\theta}_{i-1}, -\vec{\theta}_i, \vec{\theta}_{i+1}, \dots, \vec{\theta}_d)^\top\), i.e., \(\vec{\theta}'\) is the same as \(\vec{\theta}\) except that the \(i\)-th coordinate is negated.
  It is easy to verify that \(\mathcal{E}_{\vec{\theta},i}^c = \mathcal{E}_{\vec{\theta}',i}\).
  Thus, applying \Cref{lemma:bretagnolle-huber} and \Cref{lemma:divergence-decomposition}, we obtain
  \begin{align*}
    &p_{\vec{\theta},i}+p_{\vec{\theta}',i}
    \geq\frac{1}{2} \exp\ab(D(\PP_{\vec{\theta}} \parallel \PP_{\vec{\theta}'}))\\
    =\ &\frac{1}{2}\exp\ab(-\frac{1}{2}\sum_{t=1}^T \ab(\E\nolimits_{\vec{\theta}}\ab[\inprod{\vec{A}_t}{\vec{\theta}-\vec{\theta}'}^2] + \E\nolimits_{\vec{\theta}}\ab[\inprod{\vec{K}_t}{\vec{\theta}-\vec{\theta}'}^2]))\\
    =\ &\frac{1}{2}\exp(-4),
    \end{align*}
    where the last equality follows from a straightforward calculation showing that \(\inprod{\vec{A}_t}{\vec{\theta}-\vec{\theta}'}=\inprod{\vec{A}_t}{\vec{\theta}-\vec{\theta}'}=4/T\).

    Since \(|\Theta|=2^d\), we have
    \begin{align*}
        \sum_{\vec{\theta} \in \Theta} \frac{1}{|\Theta|}\sum_{i=1}^d p_{\vec{\theta},i} &= \frac{1}{|\Theta|}\sum_{i=1}^d \sum_{\vec{\theta} \in \Theta} p_{\vec{\theta},i}\\
        &= \frac{1}{2^d} \cdot d \cdot \frac{2^d}{2} \cdot \frac{1}{2}\exp(-4) = \frac{d}{4}\exp(-4).
    \end{align*}
    This implies the existence of some \(\vec{\theta}^*\in \Theta\) such that
    \begin{equation}\label{eq:sum-of-p-theta-i}
      \sum_{i=1}^d p_{\vec{\theta}^*,i}\geq \frac{d}{4}\exp(-4).
    \end{equation}
    Choosing this \(\vec{\theta}^*\) and defining the optimal arm \(\vec{a}^*\) as:
    \[\vec{a}^* = \argmax_{\vec{a} \in \mathcal{A}}\inprod{\vec{a}}{\vec{\theta}^*} = \argmax_{\vec{a} \in \mathcal{A}} \sum_{i=1}^d \vec{a}^*_i \vec{\theta}^*_i.\]
    It is easy to verify that to maximize \(\sum_{i=1}^d \vec{a}^*_i \vec{\theta}^*_i\), we must have \(a_i^* = \sign(\vec{\theta}^*_i)\) for all \(i \in [d]\).
    Therefore, the expected regret is at least
    \begin{align*}
    &\E\nolimits_{\vec{\theta}^*}\ab[R(T)] = \E\nolimits_{\vec{\theta}^*}\ab[\sum_{t=1}^T \inprod{\vec{a}^*-\vec{A}_t}{\vec{\theta}^*}]\\
    =\ & \E\nolimits_{\vec{\theta}^*}\ab[\sum_{t=1}^T \sum_{i=1}^d(\vec{a}^*_i-\vec{A}_{ti})\vec{\theta}^*_i]\\
    =\ & \E\nolimits_{\vec{\theta}^*}\ab[\sum_{t=1}^T \sum_{i=1}^d(\sign(\vec{\theta}^*_i)-\vec{A}_{ti})\vec{\theta}^*_i]\\
    =\ & \E\nolimits_{\vec{\theta}^*}\ab[\sum_{t=1}^T \sum_{i=1}^d 2\1\ab\{\sign(\vec{A}_{ti})\neq \sign(\vec{\theta}^*_i)\}\sqrt{\frac{1}{T}}]\\
    =\ & 2\sqrt{\frac{1}{T}}\sum_{i=1}^d \E\nolimits_{\vec{\theta}^*}\ab[\sum_{t=1}^T\1\ab\{\sign(\vec{A}_{ti})\neq \sign(\vec{\theta}^*_i)\}]\\
    \geq\ & \sqrt{T} \sum_{i=1}^d \PP_{\vec{\theta}^*}\ab[\sum_{t=1}^T\1\ab\{\sign(\vec{A}_{ti})\neq \sign(\vec{\theta}^*_i)\} \geq T/2] \numberthis\label{eq:apply-markov}\\
    =\ & \sqrt{T}\sum_{i=1}^d p_{\vec{\theta}^*,i} \geq \frac{\exp(-4)}{4} d\sqrt{T},
  \end{align*}
  where \Cref{eq:apply-markov} uses Markov's inequality, and the last inequality follows from \Cref{eq:sum-of-p-theta-i}.
\end{proof}

\restatelowerbound*
\begin{proof}
  Suppose we have \(\beta=\frac{d}{m}\) smaller problem instances \(I_1. I_2, \dots, I_{\beta}\), each corresponding to an \(m\)-dimensional, \(K\)-armed bandit instance with a horizon of \(T/\beta\) and we assume they have preference vectors \(\vec{\theta}_1, \dots, \vec{\theta}_{\beta} \in \RR^m\), respectively.
  We denote the arm set for instance \(I_j\) as \(\mathcal{A}^{I_j} \subset \RR^m\), and the regret incurred by instance \(I\) under policy \(\pi\) as \(R_I^{\pi}(T)\) .
  Next, we construct a \(d\)-dimensional instance \(I=(I_1, I_2, \dots, I_{\beta})\) by leting the unknown preference vector for instance \(I\) be \(\vec{\theta} = (\vec{\theta}_1^\top, \dots, \vec{\theta}_{\beta}^\top)^\top\), and dividing the time horizon \(T\) into \(\beta\) consecutive periods, each of length \(T/\beta\).
  For each time step \(t \in [T]\), the feature vectors of arms \(\mathcal{A}_t\) are constructed from instance \(I_j\), where \(j=\lceil t\beta/T \rceil\).
  Specifically, \(\mathcal{A}_t=\ab\{(\vec{0}^\top, \dots, \vec{x}^\top, \dots, \vec{0}^\top)^\top\}_{\vec{x}\in \mathcal{A}^{I_j}}\), where the non-zero entry is located at the \(j\)-th block.
  This means that at time \(t\), the learner can only get information about the \(j\)-th block of the preference vector \(\vec{\theta}\).
  Therefore for any policy \(\pi\), there exists policies \(\pi_1, \dots, \pi_{\beta}\) such that \(R_{I}^{\pi}(T) = \sum_{j=1}^{\beta}R_{I_j}^{\pi_j}(\frac{T}{\beta})\).
  Applying \Cref{lemma:lower-bound-infinite-arms}, we can always find instances \(I_1, I_2, \dots, I_{\beta}\) such that
  \begin{align*}
      R_{I}^{\pi}(T) &= \sum_{j=1}^{\beta}R_{I_j}^{\pi_j}(\frac{T}{\beta}) \geq \sum_{j=1}^{\beta} \Omega\ab(m\sqrt{\frac{T}{\beta}}) = \Omega\ab(m\sqrt{T\beta})\\
      &= \Omega\ab(m\sqrt{T\frac{d}{m}}) = \Omega\ab(\sqrt{dTm}) = \Omega\ab(\sqrt{dT\log(T)}).\tag*{\qedhere}
  \end{align*}
\end{proof}

\subsection{Technical Inequalities}
We present the technical inequalities used throughout the proofs. We provide detailed references for readers' convenience.

  \begin{lemma}[\protect\citet{bretagnolle-huber-1978}]\label{lemma:bretagnolle-huber}
    Let \(P\) and \(Q\) be probability measures on the same measurable space \((\Omega, \mathcal{F})\), and let \(A \in \mathcal{F}\) be an arbitrary event. Then,
    \[P(A) + Q(A^c) \geq \frac{1}{2} \exp(-D(P \parallel Q)),\]
    where \(D(P \parallel Q)=\int_{\Omega}\log\left(\odv{P}{Q}\right)\odif{P} = \E_P\left[\log\odv{P}{Q}\right]\) is the KL divergence between \(P\) and \(Q\). \(A^c = \Omega \setminus A\) is the complement of \(A\).
  \end{lemma}

  \begin{lemma}[KL divergence between Gaussian distributions]\label{lemma:kl-divergence-of-gaussian}
    If \(P \sim \mathcal{N}(\mu_1, \sigma^2)\) and \(Q \sim \mathcal{N}(\mu_2, \sigma^2)\), then
    \[D(P \parallel Q)=\frac{(\mu_1-\mu_2)^2}{2\sigma^2}.\]
  \end{lemma}

\begin{lemma}[Matrix Chernoff, Corollary 5.2 in \citet{tropp-2011-user-friendly}]\label{lemma:matrix-chernoff}
  Consider a finite sequence \(\set{\vec{X}_k}\) of independent, random, self-adjoint matrices with dimension \(d\).
  Assume that each random matrix satisfies
  \[\vec{X}_k \succeq \vec{0} \quad \text{and} \quad \lambda_{\max}(\vec{X}_k) \leq R \quad\text{almost surely.}\]
  Define
  \[\vec{Y}:=\sum_{k} \vec{X}_k \quad \text{and} \quad \mu_{\min} := \lambda_{\min}\ab(\E[\vec{Y}]) = \lambda_{\min}\ab(\sum_{k} \E[\vec{X}_{k}]).\]
  Then, for any \(\delta \in (0,1)\),
  \[\Pr\left[\lambda_{\min}\ab(\sum_{k} \vec{X}_k) \leq (1-\delta)\mu_{\min}\right] \leq d\ab[\frac{e^{-\delta}}{(1-\delta)^{1-\delta}}]^{\mu_{\min}/R}.\]
\end{lemma}

\begin{lemma}[Determinant-trace inequality, Lemma 10 in \citet{abbasi2011improved}]\label{lemma:det-trace-inequality}
Suppose \(\vec{X}_1, \vec{X}_2, \dots, \vec{X}_t \in \RR^d\) and for any \(1\leq s\leq t\), \(\|\vec{X}_s\|_{2} \leq L\). Let \(\overline{\vec{V}}_t = \lambda \vec{I} + \sum_{s=1}^{t} \vec{X}_s \vec{X}_s^{\top}\) for some \(\lambda > 0\). Then,
\[\det(\overline{\vec{V}}_t) \leq \ab(\lambda+\frac{tL^2}{d})^d.\]
\end{lemma}

\end{document}